\documentclass{sig-alternate-2013}

\usepackage{times}
\usepackage{epsfig}
\usepackage{graphicx,subfigure}
\usepackage{algorithm}
\usepackage{algorithmic}

\newtheorem{theorem}{Theorem}
\newtheorem{lemma}[theorem]{Lemma}

\newtheorem{corollary}[theorem]{Corollary}

\newtheorem{example}{Example}

\newtheorem{remark}{Remark}

\def\A{{\bf A}}
\def\a{{\bf a}}
\def\B{{\bf B}}

\def\C{{\bf C}}

\def\D{{\bf D}}

\def\e{{\bf e}}

\def\I{{\bf I}}

\def\LL{{\bf L}}
\def\M{{\bf M}}

\def\R{{\bf R}}

\def\S{{\bf S}}

\def\U{{\bf U}}
\def\u{{\bf u}}
\def\V{{\bf V}}
\def\v{{\bf v}}
\def\W{{\bf W}}

\def\X{{\bf X}}
\def\x{{\bf x}}
\def\Y{{\bf Y}}

\def\0{{\bf 0}}
\def\1{{\bf 1}}

\def\JM{{\mathcal J}}

\def\NM{{\mathcal N}}
\def\OM{{\mathcal O}}
\def\PM{{\mathcal P}}

\def\WM{{\mathcal W}}

\def\RB{{\mathbb R}}

\def\EB{{\mathbb E}}

\def\hmu{\hat{\mu}}

\def\Si{\mbox{\boldmath$\Sigma$\unboldmath}}

\def\Lam{\mbox{\boldmath$\Lambda$\unboldmath}}
\def\De{\mbox{\boldmath$\Delta$\unboldmath}}

\def\sgn{\mathrm{sgn}}

\def\rk{\mathrm{rank}}

\def\st{\mathsf{s.t.}}

\def\trim{\mathrm{Trim}}
\def\poly{\mathrm{poly}}

\newfont{\mycrnotice}{ptmr8t at 7pt}
\newfont{\myconfname}{ptmri8t at 7pt}
\let\crnotice\mycrnotice%
\let\confname\myconfname%

\begin{document}

\title{Adjusting Leverage Scores by Row Weighting: A Practical Approach to Coherent Matrix Completion}
%
%
%
%
%

\numberofauthors{3} 
%
\author{
%
%
\alignauthor
Shusen Wang \\
       \affaddr{College of Computer Science and Technology}\\
       \affaddr{Zhejiang University}\\
       \email{wss@zju.edu.cn}
\alignauthor
Tong Zhang \\
       \affaddr{Department of Statistics}\\
       \affaddr{Rutgers University}\\
       \email{tzhang@stat.rutgers.edu}
\and
\alignauthor
Zhihua Zhang \\
       \affaddr{Department of Computer Science and Engineering}\\
       \affaddr{Shanghai Jiao Tong University}\\
       \email{zhihua@sjtu.edu.cn}
}

\maketitle
\begin{abstract}
Low-rank matrix completion is an important problem with extensive real-world applications. When observations are uniformly sampled from the underlying matrix entries, existing methods all require the matrix to be incoherent. This paper provides the first working method for coherent matrix completion under the standard uniform sampling model. Our approach is based on the weighted nuclear norm minimization idea proposed in several recent work, and our key contribution is a practical method to compute the weighting matrices so that the leverage scores become more uniform after weighting. Under suitable conditions, we are able to derive theoretical results, showing the effectiveness of our approach. Experiments on synthetic data show that our approach recovers highly coherent matrices with high precision, whereas the standard unweighted method fails even on noise-free data.
\end{abstract}




\section{Introduction}

Matrix completion is a well established problem with extensive real-world applications
such as collaborative filtering \cite{dror2012yahoo,sigkdd2007netflix},
computer vision \cite{ji2010robust,wang2012colorization,hu2013fast},
localization in sensor networks \cite{so2007theory,candes2009exact}, etc.
Matrix completion is usually formulated as the penalized nuclear norm minimization model \cite{candes2009exact,mazumder2010spectral}
or its variants \cite{srebro2004maximum},
and can be efficiently solved by many algorithms \cite{zhang2012accelerated,recht2013parallel,yun2014nomad}.

Let $\M$ be an $n_1 \times n_2$ matrix with rank $k \ll \min(n_1, n_2)$.
We observe only a portion of the entries of $\M$ and seek to recover $\M$ based on the incomplete observations.
Let $\Omega \subset [n_1]\times [n_2]$ be an index set such that $(i,j) \in \Omega$ if the $(i,j)$-th entry of $\M$ is observed,
and let $\PM_{\Omega} (\M)$ be an $n_1 \times n_2$ matrix with $\big[\PM_{\Omega} (\M)\big]_{i j} = M_{i,j}$ for all $(i,j) \in \Omega$
and  $\big[\PM_{\Omega} (\M)\big]_{i j} = 0$ for all $(i,j) \not\in \Omega$.
In order to recover $\M$ based on the partial observation,
the cardinality of $\Omega$ must be greater than some factor.

It was shown in
\cite{candes2010power} that under {\it the uniform sampling model},
that is, each entry of $\M$ is observed independently with the same probability,
the sample complexity must be greater than $c n k (\mu+\nu) \log n $ for some constant $c$ in order to exactly recover $\M$.
Here $n=n_1 + n_2$, and $\mu \in [1, \frac{n_1}{k}]$ and $\nu \in [1, \frac{n_2}{k}]$ are the row and column matrix coherence of $\M$, respectively.
When the matrix coherence is as large as $\mu + \nu = \Theta\big( \frac{n}{k \log n} \big)$,
it is simply impossible to directly complete the matrices.

Some recent work attempts to alleviate the matrix coherence requirements.
\cite{krishnamurthy2013low} proposed an active learning approach which allows the row space to be  coherent.
They use adaptive sampling to select columns followed by accessing all entries in the selected columns.
However, their active learning setting is not suitable for most real-world problems
because it is usually impossible to access all the missing entries of a column.
For example, it is impossible to demand all users to rate one specific item or to pay one user to rate every item in the system.

In another recent work \cite{chen2014coherent}, the matrix coherence requirement is eliminated by assuming that
each entry be observed independently with probability proportional to the sum of its row and column leverage scores.
Obviously, the assumption is much more restrictive than the uniform sampling assumption used in the
previous work of \cite{candes2010power,keshavan2010matrix,candes2011robust,recht2011simpler}.

The results of \cite{chen2014coherent} can be used in the reverse direction: one may adjust the leverage scores to align with a given set of observations using appropriate weighting described below.
Suppose the $(i,j)$-th entry of $\M$ is observed with probability $p_{i j}$.
Let $\R \in \RB^{n_1 \times n_1}$ and $\C \in \RB^{n_2 \times n_2}$ be diagonal weight matrices with positive diagonal entries.
Instead of directly completing the matrix $\PM_{\Omega} (\M)$,
one may first compute $\R$ and $\C$ and then complete the matrix $\R \PM_{\Omega} (\M) \C = \PM_{\Omega} (\R \M \C)$.
The column and row weighting obviously changes the leverage scores of $\M$.
After weighting, if the sum of the $i$-th row and the $j$-th column leverage scores of $\R \M \C$
is proportional to $p_{i j}$,
then the dependence of the sample complexity on matrix coherence is eliminated.
As a special case, if the observations are uniformly sampled,
one needs to find $\R$ and $\C$ such that the leverage scores of $\R \M \C$ are as uniform as possible.
This is the idea that motivates our work.
We note that none of the previous studies  provides any satisfactory approach to find the weight matrices
$\R$ and $\C$; therefore the focus and the main contribution of this paper is a practical algorithm for
computing these matrices.

This work offers the following contributions.
\begin{itemize}
\item
    This paper provides the first practical approach to coherent matrix completion.
    Our method only makes the uniform sampling assumption, which is standard in the literature;
    our method can be potentially applied to the non-uniform sampling settings.
    Before this work, there is no way to completing coherent matrix without adding unrealistic assumptions,
    e.g.\ fulling observing a number of rows/columns \cite{krishnamurthy2013low},
    accessing any unobserved entry as one wish \cite{chen2014coherent}, etc.
\item
    To complete coherent matrix, the matrix leverage scores must be known or at least approximately estimated.
    We provide in Theorem~\ref{thm:perturbation} a way to estimate the leverage scores of $\M$ from its incompletely observed entries,
    and the estimated leverage scores are near their true values with additive-error bound.
    This result may be of independent interest.
\item
    We derived an ADMM algorithm for solving the weighted nuclear norm minimization model (\ref{eq:weightedMC}).
    Our ADMM algorithm is efficient on single machine.
\item
    We apply our method to improve another low-rank matrix recovery method called the robust principal component analysis (RPCA) \cite{candes2011robust}.
    Experiments on coherent low-rank matrix show that our proposed weighted RPCA exactly recovers the low-rank matrix from heavily noisy observation,
    whereas the standard RPCA fails on either noisy or noise-free data.
\end{itemize}

The rest of this paper is organized as follows.
Section~\ref{sec:notation} defines the notation used in this paper.
Section~\ref{sec:MC} formally describe the matrix completion problem and provides an efficient algorithm for solving the weighted nuclear norm minimization problem.
Section~\ref{sec:weight_matrix_optimization} formulates an optimization model, by solving which the weight matrices can be found and coherent matrices can be completed.
Section~\ref{sec:leverage} provides an additive-error perturbation bound for estimating the leverage scores from incompletely observed matrix
and applies this technique to perform row weighting.
Sections~\ref{sec:row_weighting} and \ref{sec:weighting_completion} devise practical algorithms for solving the model formulated in Section~\ref{sec:weight_matrix_optimization}.
Section~\ref{sec:experiments} empirically evaluates our proposed methods on several synthetic datasets.
Section~\ref{sec:rpca} applies the proposed row weighting method to improve the robust principal component analysis (RPCA) method.
The appendix is available at arXiv:1412.7938, and the MATLAB code is on the first author's home page.

\section{Notation and Preliminaries} \label{sec:notation}

Let $[m]$ denote the set $\{1, 2, \ldots, m\}$.
Given a matrix $\A$, let $\a^{(i)}$ be its $i$-th row, $\a_j$ be its $j$-th column, and $A_{i j}$ be its $(i,j)$-th entry.
Let $\I_n$ be the $n\times n$ identity matrix,
$\0$ be the all zero vector (or matrix) of the appropriate size,
and $\e_i$ be the $i$-th standard basis whose $i$-th entry is one and the remaining entries are zero.

Suppose we are given an $n_1 \times n_2$ matrix $\A$.
The singular value decomposition of $\A$ is
\vspace{-2mm}
\begin{small}
\begin{eqnarray}
\A & = & \sum_i \sigma_i (\A) \u_i \v_i^T
    \; = \; \U_{\A} \Si_{\A} \V_{\A} \nonumber \\
    & = & \underbrace{{\U_{\A,k} \Si_{\A,k} \V_{\A,k} }}_{:=\A_k}
            + {\U_{\A,-k} \Si_{\A,-k} \V_{\A,-k}}. \nonumber
\end{eqnarray}%
\end{small}%
The matrix norms are defined as follows.
Let $\|\A\|_F = \big(\sum_{i, j} A_{i j}^2\big)^{1/2}$ be the matrix Frobenius norm,
$\|\A\|_2 = \sigma_1 (\A)$ be the spectral norm,
$\|\A\|_* = \sum_i \sigma_i (\A)$ be the matrix nuclear norm,
$\|\A\|_1 = \sum_{i j} |A_{i j}|$ be the matrix $\ell_1$ norm,
and $\| \A \|_{\infty} = \max_{i,j} |A_{i j}|$ be the infinity norm.
We let the $\ell_0$ pseudo-norm $\|\A\|_0$ be the number of nonzero entries of $\A$.

The row leverage scores of $\A$ according to the best rank $k$ approximation are defined by
\vspace{-2mm}
\begin{small}
\begin{eqnarray}
\mu_i (\A_k) & = & \big( \U_{\A,k}  \U_{\A,k}^T \big)_{i i} . \nonumber
\end{eqnarray}%
\end{small}%
The row cross leverage scores are defined by
\vspace{-2mm}
\begin{small}
\begin{eqnarray}
\mu_{i l} (\A_k) & = & \big( \U_{\A,k}  \U_{\A,k}^T \big)_{i l} , \nonumber
\end{eqnarray}%
\end{small}%
The column leverage scores $\nu_j$ and cross leverage scores $\nu_{j l}$ are defined similarly by replacing $\U_{\A,k}$ by $\V_{\A,k}$.
The leverage scores and the cross leverage scores satisfy
\vspace{-2mm}
\begin{small}
\begin{align}
& \mu_i (\A_k) \; \leq \; 1
\textrm{\quad and \quad}
\sum_{i=1}^{n_1} \mu_i (\A_k) = k, \label{eq:leverage_scores} \\
\vspace{-2mm}
& \mu_i (\A_k)
\; = \;
\sum_{j=1}^{n_1} \mu_{i j}^2 (\A_k). \label{eq:cross_leverage}
\end{align}%
\end{small}%
The row matrix coherence of $\A_k$ is defined by
\vspace{-1mm}
\begin{small}
\[
\mu (\A_k) \; = \; \frac{n_1}{k} \max_i \mu_i (\A_k) \; \in \; \Big[1 , \frac{n_1}{k } \Big],
\]%
\end{small}%
and the column coherence $\nu$ is defined similarly.
In the matrix completion problem with uniform sampling,
the idealized leverage scores are $\mu_1 = \cdots = \mu_{n_1} = \frac{k}{n_1}$,
where the matrix coherence attains its minimum: $\mu = 1$.

In this paper we are particularly interested in the called rank one row weighting,
that is, scaling one row by a factor $\sqrt{1-\gamma } \in (0, 1)$.
To represent such a row weighting matrix,
we define the notation:
\begin{equation} \label{eq:def_W}
\WM({n_1}, i, \gamma) = \I_{n_1}  + \big(\sqrt{1-\gamma} - 1\big) \e_i \e_i^T,
\end{equation}
whose the $i$-th diagonal entry equals to $\sqrt{1-\gamma}$.
Suppose we are given the row leverage scores and cross leverage scores of $\M$,
then the (cross) leverage scores of $\WM({n_1}, i, \gamma) \M$ can be computed in a closed from by the following lemma.

\begin{lemma}[\cite{cohen2014uniform}] \label{lem:weighting}
Given any $\M \in \RB^{n_1 \times n_2}$ and $\gamma \in (0, 1)$,
let $\W = \WM \big(n_1, i, \gamma \big) $
be defined in (\ref{eq:def_W}).
Then the $i$-th row leverage score of $\W \M$ is
\begin{small}
\[
\mu_i (\W \M)
\; = \;
\frac{(1-\gamma) \mu_i (\M)}{1 - \gamma \mu_i (\M)}
\; \leq \;
\mu_i (\M),
\]%
\end{small}%
and the $j$-th ($j\neq i$) leverage scores are
\begin{small}
\[
\mu_j (\W\M) \; = \;
\mu_j (\M) + \frac{\gamma \mu_{ij}(\M)^2}{1 - \gamma \mu_i (\M)}
\; \geq \;
\mu_j (\M).
\]%
\end{small}%
The cross leverage scores ($j,l\neq i$) are
\begin{small}
\begin{eqnarray}
\mu_{i j} (\W \M)
& = &
\bigg( 1 + \frac{\gamma \mu_i (\M)}{1 - \gamma \mu_i (\M)} \bigg) \mu_{i j} (\M), \nonumber \\
\mu_{j l} (\W \M)
& = &
(1-\gamma) \mu_{j l} (\M) + \gamma (1-\gamma) \frac{\mu_{i j} (\M) \mu_{i l} (\M)}{1 - \gamma \mu_i (\M)} . \nonumber
\end{eqnarray}%
\end{small}%
\end{lemma}

\begin{proof}
The first two equations are given in Lemma 5 of \cite{cohen2014uniform}.
The last two equations can be proved directly by the techniques of \cite{cohen2014uniform}.
\end{proof}

Since this kind of row weighting does not change matrix rank,
we have
\begin{small}
\begin{eqnarray*}
\sum_{l=1}^{n_1} \mu_l (\W \M)
\; = \;
\rk (\W \M)
\; = \;
\rk (\M)
\; = \;
\sum_{l=1}^{n_1} \mu_l (\M).
\end{eqnarray*}%
\end{small}%
That is, the sum of leverage scores remains constant during row weighting.
Therefore, if the $i$-th row is scale,
the decrease in the $i$-th leverage score turns to the increase in the $j$-th leverage score for all $j\neq i$.

\section{Matrix Completion} \label{sec:MC}

In Section~\ref{sec:standard_nuclear_norm} we describe the standard nuclear norm minimization model for the matrix completion problem
and discuss the sample complexities under different sampling models.
In Section~\ref{sec:weighted_nuclear} we introduce the weighted nuclear norm minimization model which is explored in this paper.
In Section~\ref{sec:solution_weighted_nuclear} we provide an efficient algorithm for solving the weighted nuclear norm minimization problem;
the algorithm is described in Algorithm~\ref{alg:admm}.

\subsection{Nuclear Norm Minimization} \label{sec:standard_nuclear_norm}

Let $\M$ be an $n_1 \times n_2$ matrix with rank $k \ll \min(n_1 , n_2)$, and let $n = n_1 + n_2$.
Given a partial observation $\PM_{\Omega}(\M)$,
one can solve the following nuclear norm minimization problem to obtain a low-rank matrix $\LL^\star$ which approximates $\M$.
\begin{small}
\begin{eqnarray} \label{eq:MC}
\min_{\LL}  \; \big\| \LL \big\|_* ;\qquad
\st         \; L_{i j} = M_{i j} \quad \textrm{for all} \; (i,j) \in \Omega .
\end{eqnarray}%
\end{small}%
This model has been well studied in the literature.

Under the uniform sampling model,
\cite{candes2010power} showed that $|\Omega| = \OM(n k (\mu+\nu)^2 \log^6 n)$
or $|\Omega| = \OM(n (\mu+\nu)^4 \log^2 n)$ will be sufficient for the exact recovery of $\M$ with high probability
(when entries of $\M$ are not corrupted with noise).
Later on, \cite{recht2011simpler} improved the sample complexity to $|\Omega| = \OM\big(\max\{\tau, \mu+\nu\} n k \log^2 n)$,
where $\tau = \frac{n_1 n_2}{k} \|\U_{\M,k} \V_{\M,k}^T \|_{\infty}^2$.

If the observations are non-uniformly sampled, with
the $(i,j)$-th entry of $\M$ sampled according to a probability $p_{i j} \propto \mu_i + \nu_j$,
then $|\Omega| = \OM (n k \log^2 n)$ will be sufficient for exact recovery \cite{chen2014coherent}.
Under this setting, the sample complexity does not depend on the matrix coherence.

\subsection{Weighted Nuclear Norm Minimization} \label{sec:weighted_nuclear}

A generalization of the standard nuclear norm minimization was proposed in \cite{salakhutdinov2010collaborative,foygel2011learning} for matrix completion, known as the weighted nuclear norm minimization model:
\begin{small}
\begin{eqnarray} \label{eq:weightedMC}
\min_{\LL}  \; \big\| \R \LL \C \big\|_*; \quad
\st         \; L_{i j} = M_{i j} \quad \textrm{for all} \; (i,j) \in \Omega ,
\end{eqnarray}%
\end{small}%
where $\R$ and $\C$ are diagonal matrices.

This model was proposed in \cite{foygel2011learning} to tackle the matrix completion problem
when the elements in $\Omega$ are not uniformly sampled from $[n_1]\times [n_2]$;
that is, $M_{i j}$ is observed with probability $p_{i j} $ which is not necessary that $p_{ij}= \frac{|\Omega|}{n_1 n_2}$.
In \cite{foygel2011learning} the authors set $R_{i i}^2 = \sum_j p_{ij} := p_i^R$ and $C_{j j}^2 = \sum_i p_{i j} := p_j^C$
to compensate the unbalance in the observation.

Unfortunately, this kind of weighting does not help completing coherent matrices.
Consider a highly coherent matrix whose entries are observed uniformly at random.
In this example, the weight matrices used in \cite{foygel2011learning} should be $\R = \sqrt{n_2 p} \I_{n_1}$ and $\C = \sqrt{n_2 p} \I_{n_1}$.
Thus, \eqref{eq:weightedMC} degenerates to \eqref{eq:MC}.

In fact, naively setting $\R$ and $\C$ according to the sampling probabilities $\{p_i^R\}$ and $\{p_j^C\}$ is not the correct way of weighting.
Very recently, \cite{chen2014coherent} showed that $\R$ and $\C$
should be chosen such that the leverage scores of $\R \M \C$ are aligned with the non-uniform observations;
that is, $\mu_i (\R \M \C) + \nu_j (\R \M \C)$ should be proportional to $p_{i j}$.
However, how to compute such weight matrices remains unresolved.
We observe that simply setting $R_{ii} = \max \big\{ \mu_i^{-\alpha} (\M), \, \beta \big\}$
(and similarly for $C_{jj}$) and tuning $\alpha$ and $\beta$ does not work.
We thus need to devise a more sophisticated method to find $\R$ and $\C$, which is the main focus of the paper.

\begin{algorithm}[tb]
   \caption{ADMM for the Weighted Nuclear Norm Minimization.}
   \label{alg:admm}
\begin{footnotesize}
\begin{algorithmic}[1]
   \STATE {\bf Input:} the partially observed matrix $\PM_{\Omega} (\M)$, the parameter $\lambda$, diagonal matrices $\R$ and $\C$.
   \STATE Initialize $\LL$, $\X$, $\Y$, $\rho$;
   \REPEAT
   \STATE For all $(i,j)\in \Omega$ update $L_{i j}$ by
            \begin{footnotesize}
            \[
            L_{i j} \longleftarrow \big(1+ \rho R_{ii}^2 C_{jj}^2\big)^{-1} \big(M_{ij} + \rho R_{ii} C_{jj} X_{i j} - R_{ii} C_{jj} Y_{i j} \big);
            \vspace{-1mm}
            \]
            \end{footnotesize}
   \STATE For all $(i,j)\not\in \Omega$ update $L_{i j}$ by
            \[
            L_{i j} \longleftarrow R_{i i}^{-1} C_{j j}^{-1} \big( X_{i j} - \rho^{-1} Y_{i j} \big);
            \vspace{-1mm}
            \]
   \STATE $[\U, \Si, \V]  \longleftarrow $ SVD$\big( {\rho}^{-1}\Y + \R \LL \C \big)$;
   \STATE Update $\X$ by
            \[
            \X \longleftarrow \U \big[ \Si - \rho^{-1} \lambda  \I \big]_+ \V^T;
            \vspace{-1mm}
            \]
   \STATE $\Y \longleftarrow \Y + \rho (\R \LL \C - \X)$;
   \UNTIL{converged}
   \STATE {\bf Output:} $\LL$.
\end{algorithmic}
\end{footnotesize}
\end{algorithm}

\subsection{Solving the Weighted Nuclear Norm Minimization} \label{sec:solution_weighted_nuclear}

In practice, since the observation is perturbed by noise,
it is better to use the following regularized alternative of (\ref{eq:weightedMC}):
\begin{small}
\begin{eqnarray} \label{eq:weightedMC_noise}
\min_{\LL}  & \frac{1}{2}\| \PM_{\Omega} (\LL) - \PM_{\Omega} (\M) \big\|_F^2 + \lambda \big\| \R \LL \C \big\|_* .
\vspace{-2mm}
\end{eqnarray}%
\end{small}%
Let $\U$ and $\V$ be two matrices of sizes $n_1 \times r$ and $n_2 \times r$, respectively.
It is well known that when $r = \min \{ n_1 , n_2 \}$, we have that
\begin{small}
\begin{eqnarray}
\big\| \R \LL \C \big\|_*
\; = \;
\min_{\LL = \U \V^T} \frac{1}{2}\Big \{ \big\| \R \U \big\|_F^2 + \big\| \C \V \big\|_F^2 \Big\}.\nonumber
\end{eqnarray}%
\end{small}%
Following the max-margin matrix factorization model \cite{srebro2004maximum,rennie2005fast},
\cite{salakhutdinov2010collaborative} proposed the following optimization problem to replace (\ref{eq:weightedMC_noise}):
\begin{footnotesize}
\begin{eqnarray} \label{eq:weightedMC_MMMF}
\min_{\U, \V}  \;
\frac{1}{2}\| \PM_{\Omega} (\U \V^T) - \PM_{\Omega} (\M) \big\|_F^2 +
\frac{\lambda}{2} \Big( \big\| \R \U \big\|_F^2 + \big\| \C \V \big\|_F^2 \Big).
\end{eqnarray}%
\end{footnotesize}%
This model can be solved in the same way as the standard max-margin matrix factorization model
by algorithms such as coordinate descent \cite{yu2012scalable},
alternating least square \cite{zhou2008large,jain2013low},
stochastic gradient descent \cite{gemulla2011large}, etc.

Model (\ref{eq:weightedMC_MMMF}) naturally admits parallel computing algorithms which are efficient on distributed systems;
however, their convergence is usually slow, which leads to poor performance on single machine.
In our off-line experiments performed on a single machine,
coordinate descent, alternating least square, and stochastic gradient descent for solving (\ref{eq:weightedMC_MMMF})
all converge slowly.
We thus employ a different algorithm to solve (\ref{eq:weightedMC_MMMF}) based on the alternating direction method with multiplier (ADMM) \cite{boyd2011distributed}.
The algorithm is described in Algorithm~\ref{alg:admm} and its derivation is left to Appendix~\ref{sec:admm}.
Empirically, on a single machine, the proposed ADMM algorithm is significantly faster than
the algorithm of \cite{salakhutdinov2010collaborative}.

\section{Computing the Weight Matrices by Optimization} \label{sec:weight_matrix_optimization}

In this section we formulate an optimization model for finding the weight matrices $\R$ and $\C$.
Since row weighting does not affect the column leverage scores,
we will discuss only row weighting in the remainder of this paper. The column weight matrix $\C$ can be similarly computed.
We propose to minimize {\it the $\ell_q$ hinge loss function} to find the diagonal weight matrix $\R$:
\begin{small}
\begin{eqnarray} \label{eq:def_lq_loss}
\min_{\0 \prec \R \preceq \I_{n_1}}
L_{\M, q} (\R) :=  \bigg( \sum_{i=1}^{n_1} \max \Big\{ \mu_i (\R \M ) - \mu_i^\star , \, 0 \Big\}^q \bigg)^{\frac{1}{q}},
\end{eqnarray}%
\end{small}%
where $\{\mu_i^\star \}$ are  defined in (\ref{eq:ideal_leverage_scores}) below.
In the following we will justify this model, devise a coordinate descent algorithm to solve this model,
and discuss the choice of $q$.

\subsection{Model Formulation}

It was shown in \cite{chen2014coherent} that  the sample complexity can be reduced if
the leverage scores are aligned with the sampling probabilities of the entries.
We seek to find the row and column weight matrices $\R$ and $\C$ by solving some optimization models such that
the leverage scores of $\R\M\C$ are aligned with the sampling probabilities.
We begin with a general non-uniform sampling model.

Let the $(i,j)$-th entry of $\M$ be observed with probability $p_{i j}$ ($\geq \min \{ n_1 , n_2 \}^{-10}$),
and  denote
\vspace{-1mm}
\begin{small}
\[
p^R_i = \sum_{j=1}^{n_2} p_{i j}
\quad\textrm{  and  }\quad
p^C_i = \sum_{i=1}^{n_1} p_{i j} .
\vspace{-1mm}
\]
\end{small}%
It was shown in \cite{chen2014coherent} that the sample complexity $|\Omega |$ does not depend on the matrix coherence parameters provided that
\begin{small}
\begin{equation} \label{eq:def_pij}
p_{i j} \; \geq \;
\min \bigg\{ c_0 \frac{(n_1 \mu_i + n_2 \nu_j ) \log^2 (n_1 + n_2)}{\min\{n_1 , n_2\}} ,\; 1 \bigg\},
\end{equation}%
\end{small}%
where $c_0$ is constant.
Assume that
\[
|\Omega | = \sum_{i j} p_{i j} \geq 2 c_0 \max\{ n_1 , n_2\} k \log^2 (n_1+n_2) .
\]%
We can write inequality (\ref{eq:def_pij}) as
\[
c_1 p_{i j} \geq n_1 \mu_i + n_2 \nu_j,
\]%
summing up the two sides of which w.r.t.\ $i$ or/and $j$ yields
\begin{small}
\begin{align*}
&c_1 p_i^R \geq  n_2 (n_1 \mu_i + k ), \qquad
c_1 p_j^C \geq n_1 (n_2 \nu_j + k), \\
&c_1 = 2 n_1 n_2 k / \sum_{i j} p_{i j}.
\end{align*}%
\end{small}%
Hence, the desired leverage scores are those satisfying
\begin{eqnarray} \label{eq:ideal_leverage_scores}
\left.
  \begin{array}{c}
    \mu_i \; \leq \; \frac{2 k p_i^R}{\sum_{i=1}^{n_1} p_i^R}   - \frac{k}{n_1} \; := \; \mu_i^\star, \\
    \nu_j \; \leq \; \frac{2 k p_j^C}{\sum_{j=1}^{n_2} p_j^C}   - \frac{k}{n_2} \; := \; \nu_j^\star. \\
  \end{array}
\right.
\end{eqnarray}
Notice that $\mu_i^\star$ and $\nu_j^\star$ can be less than zero,
indicating that (\ref{eq:def_pij}) does not hold even if the leverage scores are adjusted by row weighting.
For such rows or columns, we can abandon them by setting the corresponding weights $R_{ii}$ or $C_{j j}$ to be zero.

Suppose we are given $\PM_{\Omega} (\M)$. We can estimate $p_i^R$ and $p_j^C$ by counting the indices in the set $\Omega$,
and then compute $\{ \mu_i^\star \}$ and $\{ \nu_j^\star \}$ according to (\ref{eq:ideal_leverage_scores}).
We now hope to find the diagonal matrices $\R$ and $\C$ such that $\mu_i (\R \M \C)$ and $\nu_j (\R \M \C)$ are less than or equal to $\mu_i^\star $ and $\nu_j^\star $, respectively.
We expect that
\[
\mu_i (\R \M) \; = \; \mu_i (\R \M \C) \; \leq \; \mu_i^\star  \quad \textrm{for all} \; i \in [n_1].
\]
Thus, any leverage score of $\R\M$ violating the inequality should be penalized.
Based on the above discussion, we propose the optimization model (\ref{eq:def_lq_loss}) to find the weight matrix $\R$.

As a special case, if the entries of $\M$ are uniformly observed, that is, $p_{i j} = p$ for all $i$ and $j$,
then the optimal row leverage scores are all $\mu_i^\star = \frac{k}{n_1}$.
The $\ell_q$ hinge loss function becomes
\vspace{-2mm}
\begin{small}
\begin{eqnarray}\label{eq:def_lq_uniform}
L_{\M, q} (\R) \; := \; \bigg( \sum_{i=1}^{n_1} \max \Big\{ \mu_i (\R \M ) - \frac{k}{n_1} , \; 0 \Big\}^q \bigg)^{\frac{1}{q}}.
\end{eqnarray}
\end{small}

\subsection{Solving Model (\ref{eq:def_lq_loss}) by Coordinate Descent} \label{sec:coordinate_descent}

In this subsection we propose a coordinate descent algorithm to solve the model (\ref{eq:def_lq_loss}).
Here we assume that we have access to the exact leverage scores;
this is obviously unrealistic in the matrix completion problem,
but it can help us have a good understanding of our work.
In later sections we will adapt this algorithm to solve the real matrix completion problem by estimating the leverage scores from the incomplete observation of $\M$.

The algorithm begins with $\R^{(0)} = \I_{n_1}$.
In the $t$-th step, the algorithm picks one coordinate (say $i \in [n_1]$) that violates $\mu_i \big(\R^{(t-1)} \M \big) \leq \mu_i^\star$,
and scales the $i$-th diagonal entry of $\R$ by a factor $\sqrt{1-\gamma}$ where $\gamma \in (0, 1)$ can be determined by line search.
For the $\ell_1$ hinge loss function, the best step size can be computed in a closed form (see Theorem \ref{thm:steepest_descent}).

%

The coordinate descent algorithm is equivalent to a series of {\it rank one row weightings}:
\vspace{-1mm}
\begin{small}
\[
\R^{(T)} \M = \W^{(T)} \cdots \W^{(2)} \W^{(1)} \M ,
\vspace{-1mm}
\]
\end{small}%
where $\W^{(t)}$ is a diagonal matrix with all but one diagonal entries equal to one.
Lemma~\ref{lem:weighting} indicates that the leverages after rank one row weighting can be computed in a closed form.
The closed-form solution makes computation and analysis much simpler;
this is the reason why we use coordinate descent to solve  model (\ref{eq:def_lq_loss}).

Given $\mu_i (\M)$ and the desired leverage score $\mu_i'$,
we can compute a weight $\gamma$ in order that $\mu_i (\WM (n_1 , i, \gamma) \M) = \mu_i'$.
We show this in the following corollary, which follows directly from Lemma~\ref{lem:weighting}.

\begin{corollary} \label{cor:gamma}
Given any ${n_1 \times n_2}$ matrix $\M$ and any index $i \in [n_1]$.
Suppose we are given $\mu_i (\M) \in (0, 1)$ and the desired leverage score $\mu_i' \in [0, \mu_i (\M)]$.
We compute $\gamma$ by
\begin{small}
\begin{equation} \label{eq:gamma}
\gamma \; = \;
\frac{1 - \mu_i' / \mu_i (\M)}{ 1 - \mu_i'},
\end{equation}%
\end{small}%
and scale the $i$-th row of $\M$ by $\sqrt{1-\gamma}$.
Then we have that $\mu_i \big( \WM (n_1 , i, \gamma) \M \big)  = \mu_i'$.
\end{corollary}

Let $i$ be the index selected in the $t$-th step of the coordinate descent.
Analysis shows that setting such a $\gamma$ that $\mu_i (\R^{(t)} \M) = \mu_i^\star$ leads to the steepest descent in the $\ell_1$ hinge loss function.
Therefore, if $\mu_i (\R^{(t-1)} \M) $ is known to us, we can set $\gamma$ according to the preceding corollary.
This property is summarized in the following theorem.

\begin{theorem} \label{thm:steepest_descent}
In the $t$-th step of the coordinate descent,
suppose that the index $i$ is selected and that $\mu_i (\R^{(t-1)} \M)$ is known to us.
Then scaling the $i$-th row by $\sqrt{1-\gamma^{(t)}}$ leads to the steepest descent in the $\ell_1$ hinge loss function,
where $\gamma^{(t)}$ is computed by
\[
\gamma^{(t)} \; = \;
\frac{1 - \mu_i^\star / \mu_i (\R^{(t-1)} \M)}{ 1 - \mu_i^\star},
\]
\end{theorem}

\subsection{Comparing among Different Loss Functions}

Now we discuss how to set $q$ in the the optimization problem (\ref{eq:def_lq_loss}).
We compare $\ell_\infty$, $\ell_2$, and $\ell_1$ hinge loss functions and conclude that
the $\ell_1$ hinge loss function may be the best choice among the three for two reasons.
First, optimizing the $\ell_1$ hinge loss function admits provable descent in the objective function.
Second, optimizing the $\ell_1$ function leads to big step size,
optimizing the $\ell_2$ function usually leads to small step size,
and optimizing the $\ell_\infty$ function can get stuck in some very bad configurations.
The detailed discussion and empirical comparisons are presented in Appendix~\ref{sec:compare_lq}.
In the rest of this paper, we consider optimizing the $\ell_1$ hinge loss function.

\section{Leverage Score Estimation and Inexact Row Weighting} \label{sec:leverage}

The essence of the proposed coordinate descent algorithm is
to find a series of weights $\gamma^{(1)}, \cdots , \gamma^{(T)}$ and
scale the selected rows accordingly.
However, to compute the weights, we need to know the leverage scores of $\R^{(0)} \M , \cdots, \R^{(T-1)} \M $, at lease approximately.
In Section~\ref{sec:leverage_perturbation} we provide a provable approach to estimate the leverage scores.
Then in Sections~\ref{sec:medium_leverage} and  \ref{sec:large_leverage}
we use the estimated leverage scores to compute the weights,
and provide bounds on the true leverage scores of the matrix after weighting.

\subsection{Additive-Error Bound on Leverage Scores} \label{sec:leverage_perturbation}

Suppose the $n_1 \times n_2$ rank $k$ matrix $\M$ is the ground truth,
and we only have a perturbed observation of $\M$, denote $\widetilde{\M}$.
The matrix perturbation theory of \cite{wedin1972perturbation} indicates that we can get a rough estimate of
$\U_\M \in \RB^{n_1 \times k}$ using the SVD of $\widetilde{\M}$.
Since the row leverage scores are the squared $\ell_2$ norms of rows of $\U_{\M} \in \RB^{n_1 \times k}$,
so it is possible to approximately compute the leverage scores of $\M$ through the SVD of $\widetilde{\M}$.

Based on the matrix perturbation theory of singular vectors,
\cite{ipsen2014sensitivity} studied the perturbation bounds of leverage scores
when the observation $\widetilde{\M}$ is the superposition of $\M$ and data noise $\De$.
\cite{ipsen2014sensitivity} also conducted empirical studies and concluded that
large leverage scores have small relative perturbation and
that small leverage scores yield large relative perturbation.
Although their results and analysis do not apply to the matrix completion problem,
their work motivates us to estimate the large leverage scores based on the observation $\PM_{\Omega} (\M)$.

Under the uniform sampling model,
we establish an additive-error perturbation bound for the (cross) leverage scores
based on the previous work of \cite{keshavan2010matrix,jain2013low}.
If we are given a partial observation of $\M$ with $|\Omega|=\OM \big(n k^2 \rho^2 \kappa^2 (\M)\big)$,
Theorem~\ref{thm:perturbation} ensures that an estimate of the leverage scores of $\M$
up to an additive-error of $\pm \frac{1}{2 \rho}$ can be obtained.
In this way, the rows with high leverage scores (i.e., greater than $\frac{1}{\rho}$)
can be identified and then scaled.
Here $\kappa (\M) = \sigma_1 (\M) / \sigma_k (\M)$ is the condition number of the rank $k$ matrix $\M$.

\begin{theorem} \label{thm:perturbation}
Let $\M$ be an $n_1 \times n_2$ matrix of rank $k$,
$n=\max\{n_1 , n_2\}$,
$\Omega$ be an index set whose elements are chosen from $[n_1] \times [n_2]$ uniformly at random,
and $\widetilde{\M} =  \trim \big( \PM_{\Omega} (\M) \big)$.
When the sample complexity satisfies
\vspace{-1mm}
\begin{small}
\[
|\Omega| \; \geq \; C \kappa^2 (\M) n k^{2} \rho^2
\vspace{-1mm}
\]
\end{small}
for a large enough constant $C$,
we have that with probability at least $1- n^{-3}$ all the leverage scores and the cross leverage scores satisfy that
\vspace{-2mm}
\begin{small}
\begin{eqnarray*}
\big| \mu_i (\M) - \mu_i ({\widetilde{\M}}_k) \big|
& \leq &
\frac{1}{2 \rho } ,\\
\big| \mu_{i j} (\M) - \mu_{i j} ({\widetilde{\M}}_k) \big|
& \leq &
\frac{1}{2 \rho},
\vspace{-2mm}
\end{eqnarray*}
\end{small}
for all $i, j \in \{1, \cdots , n_1 \}$.
\end{theorem}

In the statement of the theorem, ``$\trim \big( \PM_{\Omega} (\M) \big)$'' is the operation that sets to zero
all rows in $\PM_{\Omega} (\M)$ with degrees larger than $2|\Omega | / n_1$
and all columns in $\PM_{\Omega} (\M)$ with degrees larger than $2|\Omega | / n_2$.
\cite{keshavan2010matrix} show that the trim operation is necessary to ensure the bound on $\| \M - \widetilde{\M}\|_2$.
In practice, it is not necessary to throw away such rows and columns;
if a row has degree larger than $2|\Omega | / n_1$, we can randomly hold $|\Omega | / n_1$ entries of that row.


\subsection{Weighting Rows with the Medium-Scale Leverage Scores} \label{sec:medium_leverage}

We apply Theorem~\ref{thm:perturbation} to enable provable row weighting using the estimated leverage scores instead of the exact ones.
Suppose the leverage scores of $\M$ can be estimated within additive error $\pm \frac{1}{\rho}$.
The following theorem considers a row with estimated leverage score in this region: $\hmu_i \in \big[\frac{1}{\rho} , 1-\frac{1}{\rho} \big]$.

\begin{theorem} \label{thm:mu_medium}
Let $\M$ be an $n_1 \times n_2$ rank $k$ matrix which is unknown to us.
The leverage scores of $\M$ are all in $(0,1)$.
Suppose we are given the estimated leverages $\hmu_1 , \cdots, \hmu_n$
such that
\vspace{-1mm}
\begin{small}
\[
\big| \mu_i (\M) - \hmu_i \big|
\; \leq \;
\frac{1}{2\rho}
\textrm{\quad for all \;} i \in [n].
\vspace{-1mm}
\]
\end{small}
Then for any row $i \in [n_i]$ whose estimated leverage is $\hmu_i \in \big[\frac{1}{\rho} , 1-\frac{1}{\rho} \big]$, we let
\vspace{-1mm}
\begin{small}
\begin{equation} \label{eq:lem:medium_leverage}
\gamma \; = \;
\frac{n_1 - 2k/\hmu_i}{n_1 - 2k}
\vspace{-1mm}
\end{equation}
\end{small}
and compute the diagonal matrix $\W = \WM (n_1 , i , \gamma)$ according to (\ref{eq:def_W}).
The true $i$-th leverage score after the rank one row weighting is
\begin{small}
\vspace{-1mm}
\[
\mu_i (\W \M)
\; \in \;
\bigg(  \frac{k}{n_1} , \: \frac{4 k}{n_1} \big(1+ o(1) \big) \bigg).
\vspace{-1mm}
\]
\end{small}
\end{theorem}
Recall from (\ref{eq:def_lq_uniform}) that under the unform sampling model,
the ideal leverage scores are $\mu_1=\cdots=\mu_{n_1} = \frac{k}{n_1}$.
Thus weighting a row according to the theorem makes its leverage score more desirable.

\subsection{Weighting Rows with the Leverage Scores Close to One} \label{sec:large_leverage}

Theorem~\ref{thm:mu_medium} indicates that when the estimated leverage score $\hmu_i \in \big[\frac{1}{\rho} , 1-\frac{1}{\rho} \big]$,
we can safely scale the $i$-th row by the factor $\sqrt{1-\gamma}$,
where $\gamma$ is a function of $\hmu_i$ defined in (\ref{eq:lem:medium_leverage}).
We plot $\hmu_i$ versus $\sqrt{1-\gamma}$ in Figure~\ref{fig:gamma_hmu}.

We can see that when $\hmu_i \in \big[\frac{1}{\rho} , 1-\frac{1}{\rho} \big]$,
a small perturbation in $\hmu_i$ leads to a small change in $\sqrt{1-\gamma}$.
However, when $\hmu_i$ is close to one, a small perturbation in $\hmu_i$ leads to a big change in $\sqrt{1-\gamma}$.
Thus, taking a big step size as in (\ref{eq:lem:medium_leverage}) can be hazardous when $\hmu_i$ is close to one.
When $\hmu_i > 1 - \frac{1}{\rho}$, we should set $\gamma$ small and gradually and safely decrease the $i$-th leverage score.
This strategy is described and analyzed in the following theorem.

\begin{theorem} \label{thm:mu_large}
Let the assumptions of Theorem~\ref{thm:mu_medium} hold except for that
$\hmu_i \in (1-\frac{1}{\rho}, 1)$.
We compute $\gamma$ by
\vspace{-2mm}
\begin{small}
\begin{equation} \label{eq:large_leverage}
\gamma
\; = \; \frac{\rho - \frac{1}{\hmu_i - 1/2\rho}}{\rho - 1},
\vspace{-1mm}
\end{equation}
\end{small}%
and let $\W = \WM (n_1 , i , \gamma)$.
Then we have $\mu_i (\W \M) \in  \big[ \frac{1}{\rho}, \, \mu_i (\M) \big)$.
\end{theorem}
By scaling the $i$-th row multiple times, we can make sure that $\hmu_i$ falls in the interval $\big[\frac{1}{\rho} , 1-\frac{1}{\rho} \big]$,
and can then take a big step size according to Theorem \ref{thm:mu_medium}.

\begin{algorithm}[tb]
   \caption{Coordinate Descent Using Estimated Leverage Scores.}
   \label{alg:weighing}
\begin{small}
\begin{algorithmic}[1]
   \STATE {\bf Input:} an $n_1 \times n_2$ matrix $\widetilde{\M}$ and a target rank $k$;
   \STATE Initialize: $\R^{(0)} = \I_{n_1}$, $\rho^{(0)} = \sqrt{\frac{|\Omega|}{C n k^2 \kappa^2 ( \M )}} \gg 1$;
   \FOR{$t=1$ to $T$}
   \STATE $\hmu_i \longleftarrow$ the leverage scores of $(\R^{(t-1)} \widetilde{\M})_k$, for all $i\in [n_1]$; \label{alg:weighing:compute_lev}
   \STATE Pick an index such that $\hmu_i \geq \frac{1}{\rho^{(t-1)}}$;  Stop if there is no such $i$;
   \STATE If $\hmu_i \in \big[\frac{1}{\rho^{(t-1)}}, 1 - \frac{1}{\rho^{(t-1)}} \big]$,
            compute $\gamma$ according to (\ref{eq:lem:medium_leverage});
   \STATE If $\hmu_i \geq 1 - \frac{1}{\rho^{(t-1)}}$,
            compute $\gamma$ according to (\ref{eq:large_leverage});
   \STATE $\W = \WM (\gamma, n, i)$;
   \STATE $\R^{(t)}  \longleftarrow \W  \R^{(t-1)} $; \quad
            $\rho^{(t)} \longleftarrow\frac{\kappa (\M)}{ \kappa ( \R^{(t)} \M )} \rho^{(0)}$;
   \ENDFOR
   \STATE {\bf Output:} $\R = \R^{(T)} $.
\end{algorithmic}
\end{small}
\end{algorithm}

\begin{figure}[htbp]
\centering
\begin{minipage}[t]{0.23\textwidth}
    \begin{flushleft}
    \includegraphics[width=0.95\textwidth]{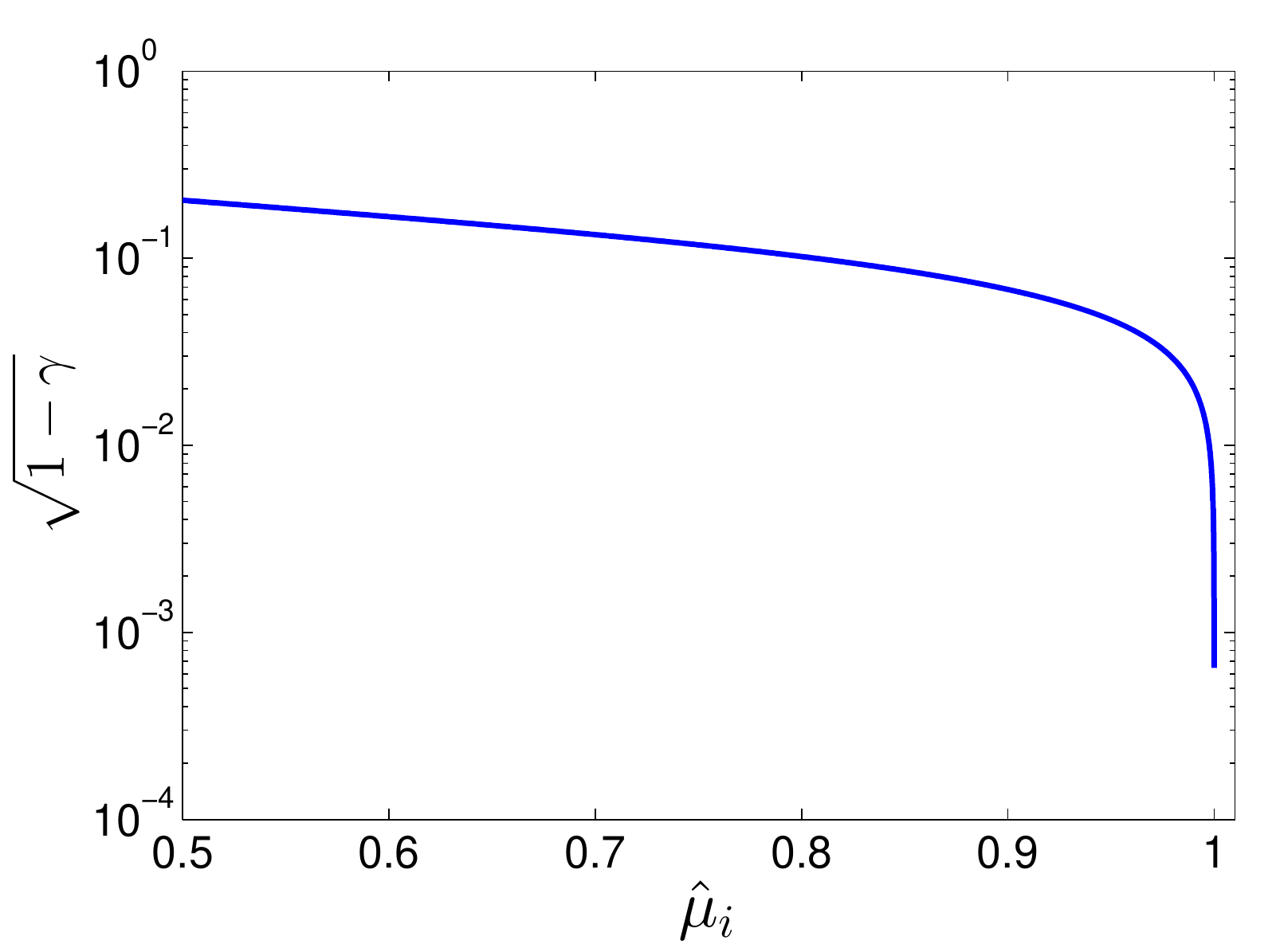}
    \caption{The plot of $\hmu_i$ versus $\sqrt{\gamma}$, where $\hmu_i$ and $\gamma$ are defined (\ref{eq:lem:medium_leverage}) with $k=20$, $n_1=1,000$, and $\rho = 20$.}
    \label{fig:gamma_hmu}
    \end{flushleft}
\end{minipage}
\begin{minipage}[t]{0.23\textwidth}
    \begin{flushright}
    \includegraphics[width=0.95\textwidth]{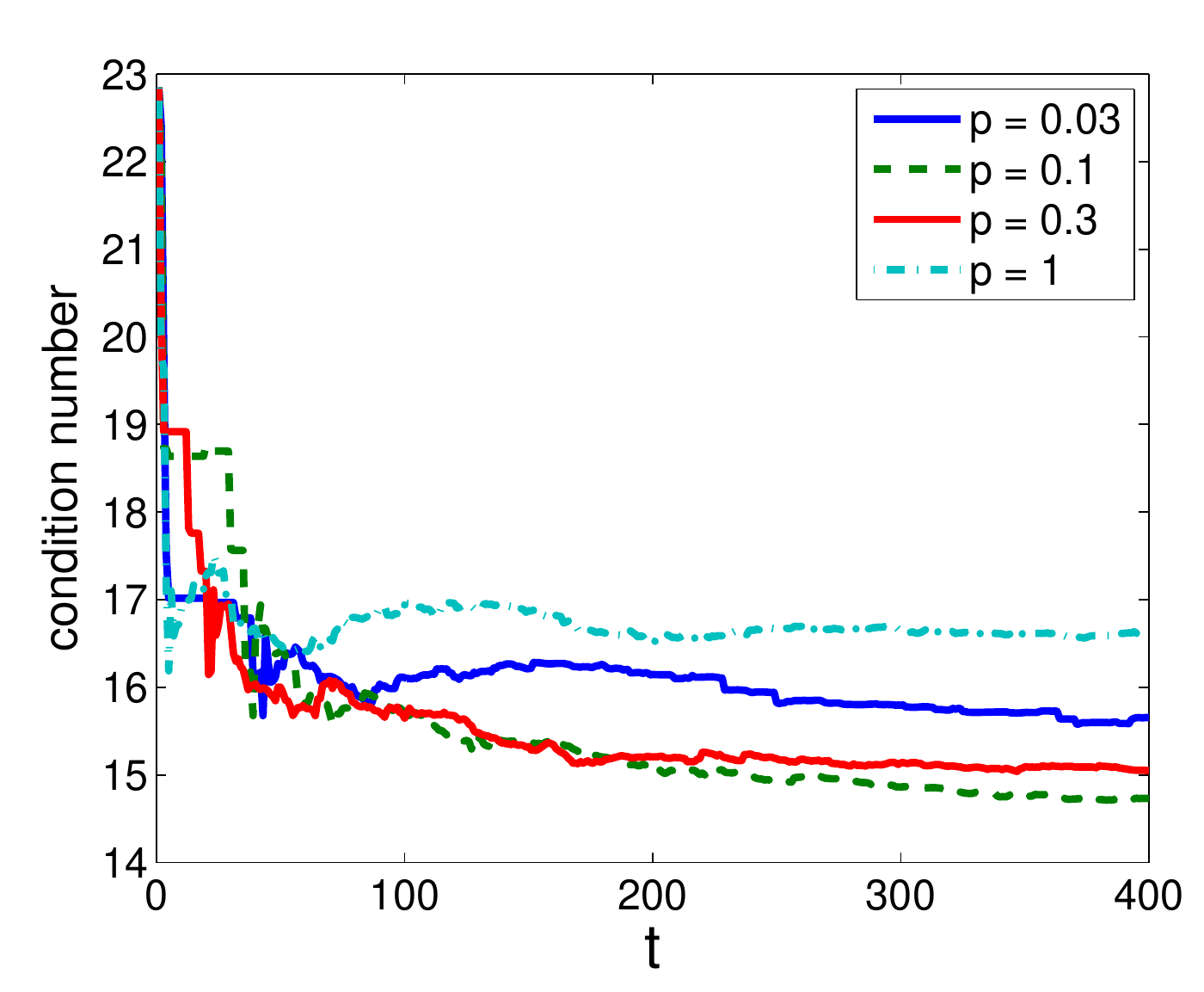}
    \caption{The change in condition number $\kappa \big( \R^{(t)} \M \big)$ during the rank one row weightings.}
    \label{fig:condnum_density}
    \end{flushright}
\end{minipage}
\end{figure}

\section{Practical Coordinate Descent Algorithm} \label{sec:row_weighting}

In this section we provide a practical coordinate descent algorithm to optimize the $\ell_1$ hinge loss function under the uniform sampling model without knowing the exact leverage scores.
The algorithm is described in Algorithm~\ref{alg:weighing}.
The algorithms proposed in this section do not directly apply to the non-uniform sampling model
because currently we do not know how to estimate the leverage scores from non-uniformly sampled entries.

\subsection{Algorithm Description}

The algorithm is based on the same idea as the one in Section~\ref{sec:coordinate_descent}
except that here we use the estimated leverage scores instead of the exact ones.
To optimize the $\ell_1$ hinge loss function in (\ref{eq:def_lq_uniform}),
in each step the algorithm picks an index $i\in [n]$ which violates $\mu_i (\R \M ) \leq \frac{k}{n_1}$.
Notice that our estimated leverage scores have $\pm \frac{1}{2\rho}$ additive error,
thus only the leverage score greater than $\frac{k}{n_1} + \frac{1}{2\rho}$ can be identified.

We let $\widetilde\M$ be the observation, which is $\PM_{\Omega} (\M)$ in the matrix completion problem.
We let $i$ be the selected index and $\hmu_i$ be the estimated leverage score of $\R^{(t-1)} \M$.
We can compute the weight $\gamma$ according to Theorems~\ref{thm:mu_medium} and \ref{thm:mu_large}
and weight the $i$-th row by $\sqrt{1-\gamma}$.
This kind of row weighting has provable bound on the true leverage scores,
and thus leads to provable decrease in the objective function of (\ref{eq:def_lq_uniform}).

\subsection{Analysis}

In this subsection we show that each iteration in Algorithm~\ref{alg:weighing}
does not increase the objective function value (\ref{eq:def_lq_uniform})
and decrease the objective function value under some conditions.

It follows from Lemma~\ref{lem:weighting} that the decrement in the $i$-th leverage score turns to the increments
in the $j$-th leverage scores for all $j \neq i$.
We denote the increments in the $j$-th ($j \neq i$) leverage scores by
\begin{small}
\begin{align} \label{eq:l1_loss_decrement1}
&\Delta \mu_{j}
\;:=\;  \mu_j \big(\R^{(t)}\M \big) - \mu_j \big(\R^{(t-1)}\M \big) \nonumber \\
& = \; \frac{\mu_{i j}^2 (\R^{(t-1)} \M)}{\sum_{l} \mu_{i l}^2 (\R^{(t-1)} \M)} \Big( \mu_i \big(\R^{(t-1)}\M\big) - \mu_i \big(\R^{(t)}\M\big) \Big) \geq 0.
\end{align}
\end{small}
We let $\JM$ be the index set
\begin{small}
\begin{align} \label{eq:l1_loss_decrement_set}
&\JM \; = \;
\Big\{  j \in [n_1] \Big| \; j \neq i
    \textrm{ and }\; \mu_{i j} \big(\R^{(t-1)} \M\big) \neq 0 \nonumber \\
&\qquad\qquad\qquad\qquad\qquad    \textrm{ and }\; \mu_j \big(\R^{(t-1)}\M\big) < \mu_j^\star \Big\}.
\end{align}
\end{small}
The the following theorem ensures that the $\ell_1$ hinge loss function value does not increase
and strictly decrease when $\JM$ is nonempty.

\begin{theorem}\label{thm:approx_cd}
In each iteration of the algorithm, the following inequality holds:
\begin{small}
\[
L_{\M,1} \big(\R^{(t)} \big) \; \leq \; L_{\M,1} \big( \R^{(t-1)} \big).
\]
\end{small}
The decrement in the objective function is
\begin{small}
\begin{align} \label{eq:l1_loss_decrement2}
&L_{\M,1} \big(\R^{(t-1)}\big) - L_{\M,1} \big(\R^{(t)}\big) \nonumber \\
& = \;
\sum_{ j \in \JM }
\min \bigg\{  \Delta \mu_j , \; \mu_j^\star - \mu_j \big(\R^{(t-1)}\M\big) \bigg\}
\:\geq \: 0.
\end{align}%
\end{small}%
Let $i \in [n_1]$ be the index selected in the $t$-th iteration.
Suppose that $\JM$ is nonempty, equivalently, at least one row (say $j\neq i$) satisfies
\begin{small}
\[
\mu_j \big( \R^{(t-1)} \M \big) < \mu_j^\star
\quad \textrm{ and } \quad
\mu_{i j} \big( \R^{(t-1)} \M \big) > 0.
\]%
\end{small}%
Then we have that
\begin{small}
\[
L_{\M,1} \big(\R^{(t)} \big) \; < \; L_{\M,1} \big( \R^{(t-1)} \big).
\]
\end{small}
\end{theorem}

The following corollary indicates that $\gamma$ is the greater the better when $\mu_i \big( \R^{(t)} \M \big) \geq \mu_i^\star$ holds.
Therefore, under the uniform sampling model, setting $\gamma$ according to Theorem~\ref{thm:mu_medium} leads
to the greatest provable decrease in the $\ell_1$ hinge loss function.

\begin{corollary}\label{cor:decrement}
Let $i \in [n_1]$ be the index selected in the $t$-th iteration, that is,
$\mu_i \big( \R^{(t-1)} \M \big) > \mu_i^\star$.
Since $\mu_i \big( \R^{(t)} \M \big)$ decreases as $\gamma$ increases,
we assume that $\gamma$ is small enough such that $\mu_i \big( \R^{(t)} \M \big) \geq \mu_i^\star$,
equivalently, assume that
\begin{small}
\begin{equation*}
\gamma \; \leq \;
\frac{1 - \mu_i^\star / \mu_i ( \R^{(t-1)} \M )}{ 1 - \mu_i^\star}.
\end{equation*}%
\end{small}%
Then the decrement in the $\ell_1$ hinge loss function increases as $\gamma$ increases.
\end{corollary}

%
%

\begin{remark}
To perform the rank one row weighting,
we need to obtain an estimate of $\mu_i \big( \R^{(t-1)} \M \big)$ within additive error $\pm \frac{1}{2\rho}$
for all $i \in [n_1]$ and $t \in [T]$.
If we use Theorem~\ref{thm:perturbation} to estimate the leverage score, the parameter $\rho$ is not constant.
In fact, in the $t$-th iteration, the parameter $\rho$ becomes
\begin{small}
\vspace{-1mm}
\[
\rho^{(t)} \; = \;
\sqrt{\frac{|\Omega|}{C n k^2 \kappa^2 \big( \R^{(t)} \M \big)}}
\; = \;
\frac{\kappa (\M)}{ \kappa \big( \R^{(t)} \M \big)} \rho^{(0)} .
\vspace{-1mm}
\]
\end{small}%
Here $n$, $k$, $\kappa$, $C$ are defined in Theorem~\ref{thm:perturbation},
and larger $\rho$ is more desirable.
Unfortunately, how to estimate the condition numbers
$\kappa\big( \R^{(t)} \M \big) = \sigma_1 \big( \R^{(t)} \M \big) \big/ \sigma_k \big( \R^{(t)} \M \big) $ remains unknown,
so we are unable to update $\rho^{(t)}$ in each iteration.
We have to empirically fix $\rho^{(0)}=\cdots =\rho^{(T-1)} = \rho$
and assume that the condition number $\kappa \big( \R^{(t)} \M \big)$ does not deteriorate
much during the sequence of rank one row weightings.
\end{remark}

In practice, the assumption that the
condition number $\kappa \big( \R^{(t)} \M \big)$ does not deteriorate holds.
In fact, in our experiments, the condition number actually gets better during the sequence of
rank one row weightings.
Figure~\ref{fig:condnum_density} shows the changes of the condition number of the matrix $\R^{(t)} \M$
under the same experimental setting as that of Section~\ref{sec:experiments:density}.

\subsection{Computational Issues}

The most expensive operation in Algorithm~\ref{alg:weighing} is computing the leverage score of $(\R^{(t-1)} \widetilde\M)_k$ for all $t \in [T]$,
which requires the rank $k$ truncated SVD of $\R^{(t-1)} \widetilde\M$ and costs time $\OM(T n_1 n_2 k)$.
The memory cost of Algorithm~\ref{alg:weighing} is $\OM \big( \|\widetilde\M \|_0 + n_1 k \big)$, which is not a big challenge.
We seek to reduce the time costs in the following ways.

If $\rk (\widetilde\M) = r$ and $k \leq r  < n_2$,
we can reduce the time cost to $\OM (n_1 n_2 r + T n_1 r k)$ in the following way.
We first compute the condensed SVD of $\widetilde\M$ to obtain its column bases $\U_{\widetilde\M} \in \RB^{n_1\times r}$ in time $\OM (n_1 n_2 r)$.
Exploiting the fact that the leverage scores of $\R \widetilde\M$ and $\R \U_{\widetilde\M}$ are the same for any nonsingular matrix $\R$
(see Theorem~\ref{thm:equivalent_leverage} in the appendix),
we replace $\widetilde\M \in \RB^{n_1 \times n_2}$ in Algorithm~\ref{alg:weighing} by the smaller matrix $\U_{\widetilde\M} \in \RB^{n_1 \times r}$.
Then in each iteration it cost only $\OM (n_1 r k)$ time to compute the leverage scores of $\R^{(t-1)} \U_{\widetilde\M}$.

If $\rk(\widetilde\M)$ is not much smaller than $n_2$, the above approach does not help accelerating computation.
Since $(\R^{(t-1)} \widetilde\M)_k$ is merely a rough approximation of $\R^{(t-1)} \M$,
it is unnecessary to compute the exact rank $k$ SVD of $(\R^{(t-1)} \widetilde\M)_k$.
Instead, we propose to compute the rank $k$ SVD of $\R^{(t-1)} \widetilde\M$ approximately using the matrix sketching techniques~\cite{woodruff2014sketching}.
For example, when $\widetilde\M = \PM_{\Omega} (\M)$, which is a highly sparse matrix,
the rank $k$ SVD of $\widetilde\M$ can be computed in time
$\OM (|\Omega| + \poly (k))$ by the sparse matrix embedding method \cite{clarkson2013low}.
In this way, the total time cost drops to $\OM \big(T |\Omega| + T \cdot \poly (k) \big)$.

\section{The Weighting---Completion Algorithm} \label{sec:weighting_completion}

In practice, we find that 
the algorithm can gradually make all the leverage score below $\frac{1}{\rho}$.
However, the algorithm can do nothing to the leverage scores between $\frac{k}{n_1}$ and $\frac{1}{\rho}$
because a very small leverage score may appear large due to the additive error perturbation.
Therefore, the row matrix coherence after the row weighting is $\frac{n_1}{\rho k }$, ideally.
If we want to attain an even lower matrix coherence,
we must acquire more accurate estimates of the leverage scores.
For this purpose, we propose a heuristic for obtaining better estimates of the leverage scores.

We can first use $\PM_{\Omega} (\M)$ to estimate the row and column leverage scores
and then perform matrix completion using the weighted matrix completion model (\ref{eq:weightedMC}) to obtain $\LL^\star$.
Empirically, compared with $\PM_{\Omega} (\M)$,
the leverage scores of $\LL^\star$ better approximates those of $\M$,
and we can thus obtain better estimation of the leverage scores of $\M$ based on $\LL^\star$.
We propose to perform once more row and column weighting using $\LL^\star$ and $k$ as the input of Algorithm~\ref{alg:weighing}.
We can also repeat this weighting---completion procedure multiple times to attain better results.
This weighting---completion procedure is described in Algorithm~\ref{alg:iterative}.

\begin{algorithm}[tb]
   \caption{The Weighting---Completion Algorithm.}
   \label{alg:iterative}
\begin{small}
\begin{algorithmic}[1]
   \STATE {\bf Input:} the partially observed matrix $\PM_{\Omega} (\M)$, a target rank $k$.
   \STATE $\widetilde{\M}^{(0)} \longleftarrow$ Trim $\PM_{\Omega} (\M)$ according to \cite{keshavan2010matrix};
   \FOR{$s=1, 2, \cdots $}
   \STATE $\R \longleftarrow$ Algorithm \ref{alg:weighing} taking ${\widetilde{\M}^{(s)}}$ and $k$ as input;
   \STATE $\C \longleftarrow$ Algorithm \ref{alg:weighing} taking $(\widetilde{\M}^{(s)})^T$ and $k$ as input;
   \STATE $\widetilde{\M}^{(s)} \longleftarrow$ weighted matrix completion taking $\PM_{\Omega} (\M)$, $\R$, $\C$ as input;
   \ENDFOR
   \STATE {\bf Output:} $\widetilde{\M}^{(s)}$.
\end{algorithmic}
\end{small}
\end{algorithm}

\section{Experiments} \label{sec:experiments}

We conduct experiments to evaluate the coordinate descent algorithm and the weighted matrix completion.
The data and algorithms are all for the uniform sampling model.

\subsection{Datasets} \label{sec:datasets}

To demonstrate the effectiveness of our approach,
we generate an $n_1\times n_2$ low-rank matrix $\LL_0 = \U \V^T$ with high row and column coherences.
We generate $\U\in \RB^{n_1 \times k}$ and $\V\in \RB^{n_2 \times k}$ with each row sampled from the multivariate $t$ distribution
with $2$ degrees of freedom and the covariance matrix $\Lam \in \RB^{k\times k}$ whose $(i,j)$-th entry is
$\Lambda_{i j} \; = \; 2 \times 0.5^{|i-j|}$.
We generate an index set $\Omega $ by sampling each element of $[n_1]\times [n_2]$ with probability $p$.
Thus, the expected number of observed entries is $\EB |\Omega| = p n_1 n_2$.
The $n_1 \times n_2$ matrix $\PM_{\Omega} (\M)$ is our observation,
where $\M = \LL_0 + \S_0$ and $\S_0$ captures the noise.

Our weighted approach does not achieve markedly advantage over the unweighted on the collaborative filtering data.
The collaborative filtering data are highly unbalanced, and thus the uniform sampling model
considered in this paper is violated.
Our proposed algorithms in Section~\ref{sec:row_weighting} does not apply to the general non-uniform sampling model because the leverage scores cannot be accurately estimated using our approach.
Our tentative experiments on synthetic data shows that naively estimating the leverage scores in the same way as the uniform sampling model does not work well:
the weighted approach does not have noticeable advantage over the unweighted approach
unless a large portion of entries were observed, e.g. $20\%$ entries were observed.
However, in the collaborative filtering problems, commonly less than $1\%$ entries are observed.

\subsection{Compared Methods} \label{sec:experiments:methods}

We use model (\ref{eq:weightedMC_noise}) under different settings of the weight matrices $\R$ and $\C$.
If $k$ (the rank of the underlying low-rank matrix) is unknown,
we treat is as a parameter and tune it.
\begin{itemize}
\vspace{-3mm}
\item {\bf Unweighted.}
    Set $\R=\I_{n_1}$ and $\C = \I_{n_2}$.
\vspace{-2mm}
\item {\bf Type 1.}
    Take $\PM_{\Omega} (\M) $ and $k$ as the input of Algorithm~\ref{alg:weighing} to compute $\R$ and $\C$,
    and solve model (\ref{eq:weightedMC_noise}) to obtain the solution $\LL^\star$.
    That is, run Algorithm~\ref{alg:iterative} with only one repeat.
\vspace{-2mm}
\item {\bf Type 2.}
    Let $\LL^\star$ be computed by Type 1.
    Take $\widetilde{\M} = \LL^\star$ and $k$ as the input of Algorithm~\ref{alg:weighing} to re-compute $\R$ and $\C$.
    That is, run Algorithm~\ref{alg:iterative} for two repeats.
    {\bf Type 3}, {\bf Type 4}, \dots, are similarly defined.
\vspace{-2mm}
\end{itemize}

For each data matrix and each method, we tune the parameter $\lambda$ to the best and report the matrix completion error
\begin{small}
\begin{equation} \label{eq:def_error}
\vspace{-1mm}
\textrm{Error} \; = \; {\|\LL^* - \LL_0 \|_F} \, \big/ \, {\|\LL_0 \|_F},
\vspace{-1mm}
\end{equation}
\end{small}%
where $\LL_0$ denotes the ground truth, i.e., the low-rank component of $\M$.

\subsection{Effects on Leverage Scores} \label{sec:experiments:density}

In this subsection we test how the leverage scores changes in each step of the coordinate descent algorithm~\ref{alg:weighing}.
Let $T$, $t$, $\rho$ be defined in Algorithm~\ref{alg:weighing},
and set $\rho = 20\sqrt{p}$ and the total iterative number of coordinate descent to be $T = k^2$.
We generate the synthetic data matrix $\LL_0$ described in Section~\ref{sec:datasets} with different settings of $n_1$, $n_2$, $k$, and $p$,
and let $\M = \LL_0$ without adding noise.

We first test the coordinate descent algorithm under different $p = |\Omega| / n_1 n_2$.
We fix $n_1 = 2,000$, $n_2 = 1,000$, $k = 20$ to generate the low rank matrix $\M$,
and vary the number of samples by setting $p=0.03$, $0.1$, $0.3$, $1$ respectively to obtain $\Omega$.
We take $\PM_{\Omega}(\M)$ and $k$ as the input of Algorithm~\ref{alg:weighing},
and report in Figure~\ref{fig:leverage_scores_density} the row coherence $\mu \big(\R^{(t)} \M\big)$
and the $\ell_1$ hinge loss function value $L_{\M, 1} \big(\R^{(t)} \big)$ in each step of the coordinate descent algorithm, respectively.

\begin{figure}[!ht]
\subfigtopskip = 0pt
\begin{center}
\centering
\subfigure[$t$ versus matrix coherence]{\includegraphics[width=0.23\textwidth]{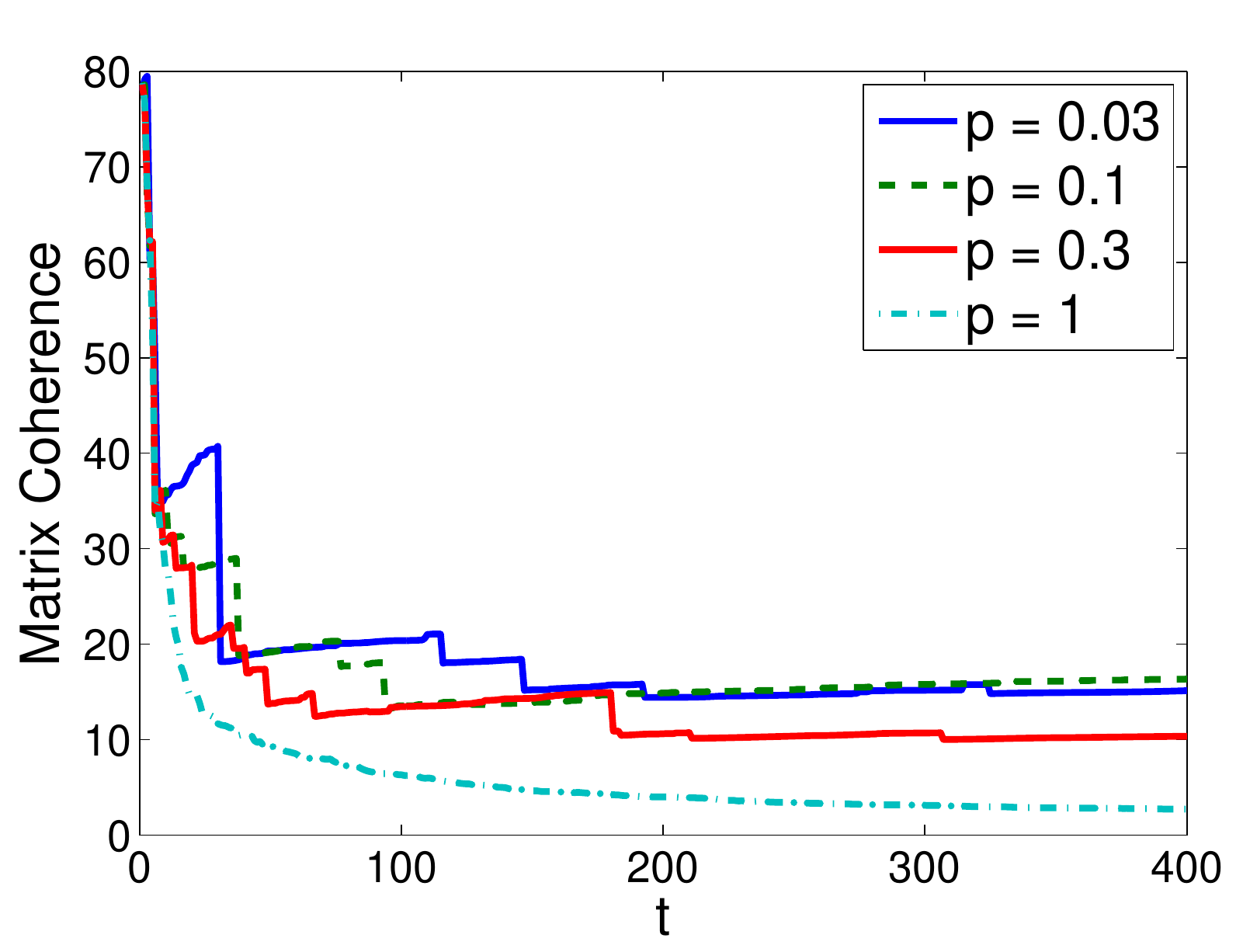}}
\subfigure[$t$ versus the $\ell_1$ hinge loss function value]{\includegraphics[width=0.23\textwidth]{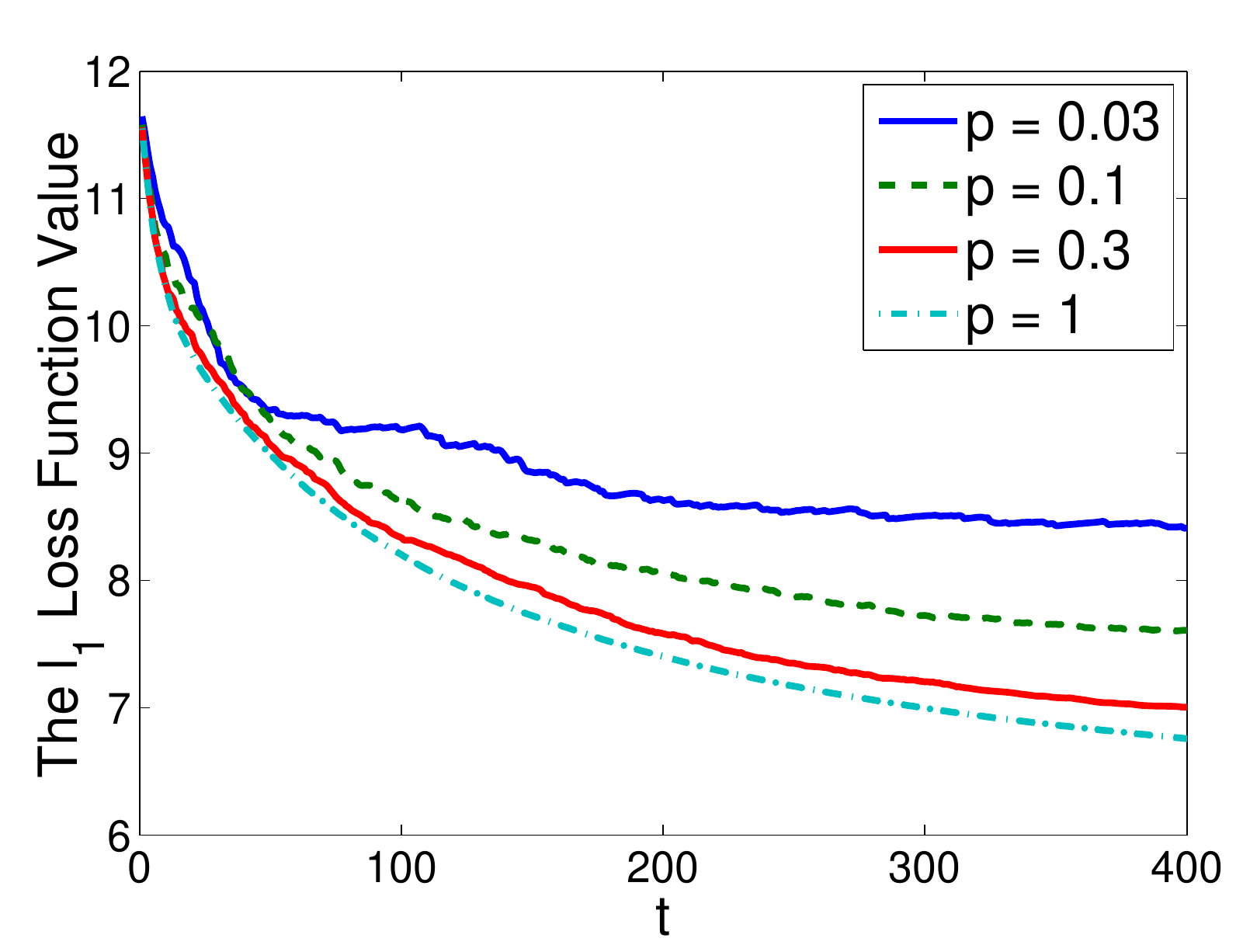}}
\end{center}
   \caption{Adjusting leverage scores by row weighting under different number of samples.}
\label{fig:leverage_scores_density}
\end{figure}

In Figure~\ref{fig:leverage_scores_density}
the plot of the $\ell_1$ hinge loss function indicates that
the leverage scores become more uniform after each rank one row weighting using Algorithm~\ref{alg:weighing}.
Theorem~\ref{thm:perturbation} indicates that larger $p$ results in better estimation of the leverage scores,
and consequently the row weighting can be better performed.
The experimental results verify our intuition.

We test the weighting---completion algorithm (Algorithm~\ref{alg:iterative})
to see whether alternating between weighting and completion helps making the leverage score more uniform.
We fix $n_1 = 2,000$, $n_2 = 1,000$, $k = 20$, $p = \frac{|\Omega|}{n_1 n_2} = 0.1$.
We run the weighting---completion procedure in Algorithm~\ref{alg:iterative} for four iterations
and plot $t$ against the the matrix coherence or the $\ell_1$ hinge loss function $L_{\M,1} (\R^{(t)})$ in Figure~\ref{fig:leverage_scores_vs_t}.

\begin{figure}[!ht]
\subfigtopskip = 0pt
\begin{center}
\centering
\subfigure[$t$ versus matrix coherence]{\includegraphics[width=0.23\textwidth]{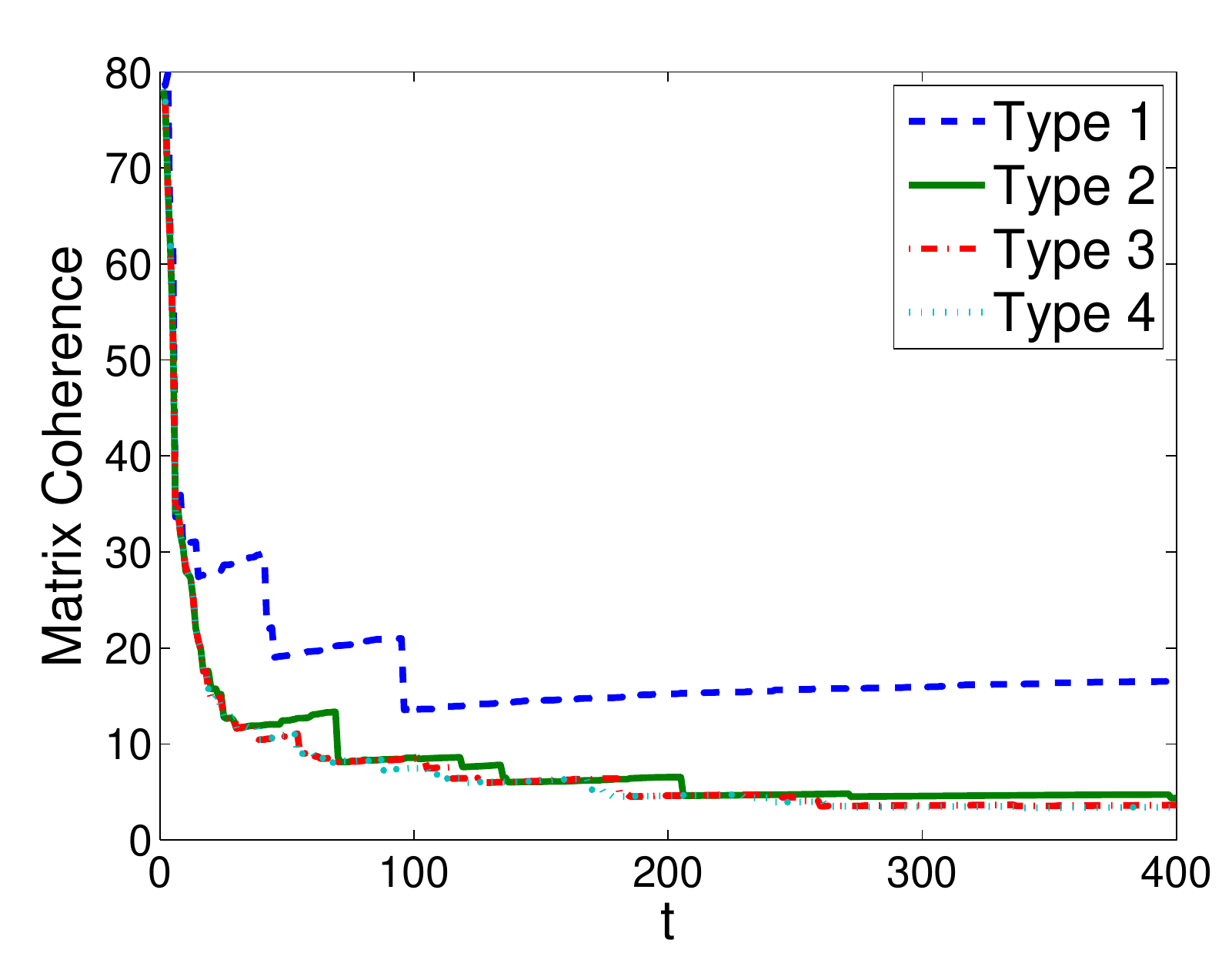}}
\subfigure[$t$ versus the $\ell_1$ hinge loss function value]{\includegraphics[width=0.23\textwidth]{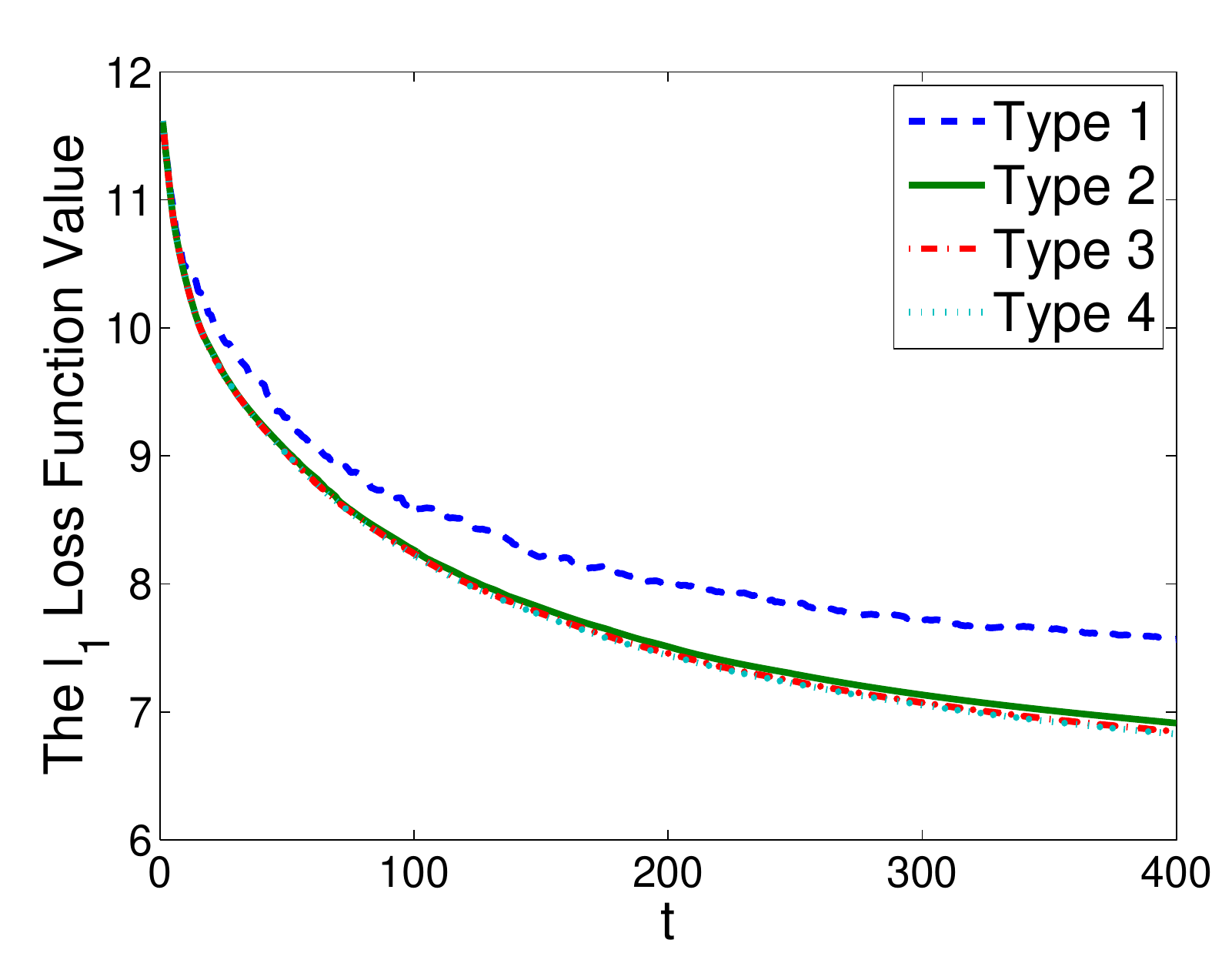}}
\end{center}
   \caption{The effects of the weighting---completion algorithm (Algorithm~\ref{alg:iterative}).}
\label{fig:leverage_scores_vs_t}
\end{figure}

Figure~\ref{fig:leverage_scores_vs_t} clearly shows that
performing the weighting---completion procedure multiple rounds results in much lower matrix coherence and the $\ell_1$ hinge loss function value
than performing the procedure only one round.
The results suggest that running the weighting---completion procedure for two rounds is a good choice.

\subsection{Matrix Completion Accuracy}

In this set of experiments, we use the synthetic data described in Section~\ref{sec:datasets} to generate $\LL_0$
and add i.i.d.\ Gaussian noise $\NM (1, \sigma^2)$ to $50\%$ entries of $\LL_0$ to obtain $\M$.
We set $n_1 = 2,000$, $n_2=1,000$, $k = 20$, and vary $p$ and $\sigma$.
We set $p=0.05$, $0.1$, or $0.2$, and vary $\sigma$ from $0$ to $40$.
For each data matrix, each $p$, $\sigma$, and each method,
we tune the parameter $\lambda$ to attain the lowest matrix completion error (defined in (\ref{eq:def_error})).
We plot in Figure~\ref{fig:noisyMC} the noise intensity $\frac{\| \PM_{\Omega} (\M - \LL_0) \|_F}{\| \PM_{\Omega} (\M) \|_F}$
against the matrix completion error.

\begin{figure}[!ht]
\subfigtopskip = 0pt
\begin{center}
\centering
\subfigure[$p=5\%$]{\includegraphics[width=0.15\textwidth]{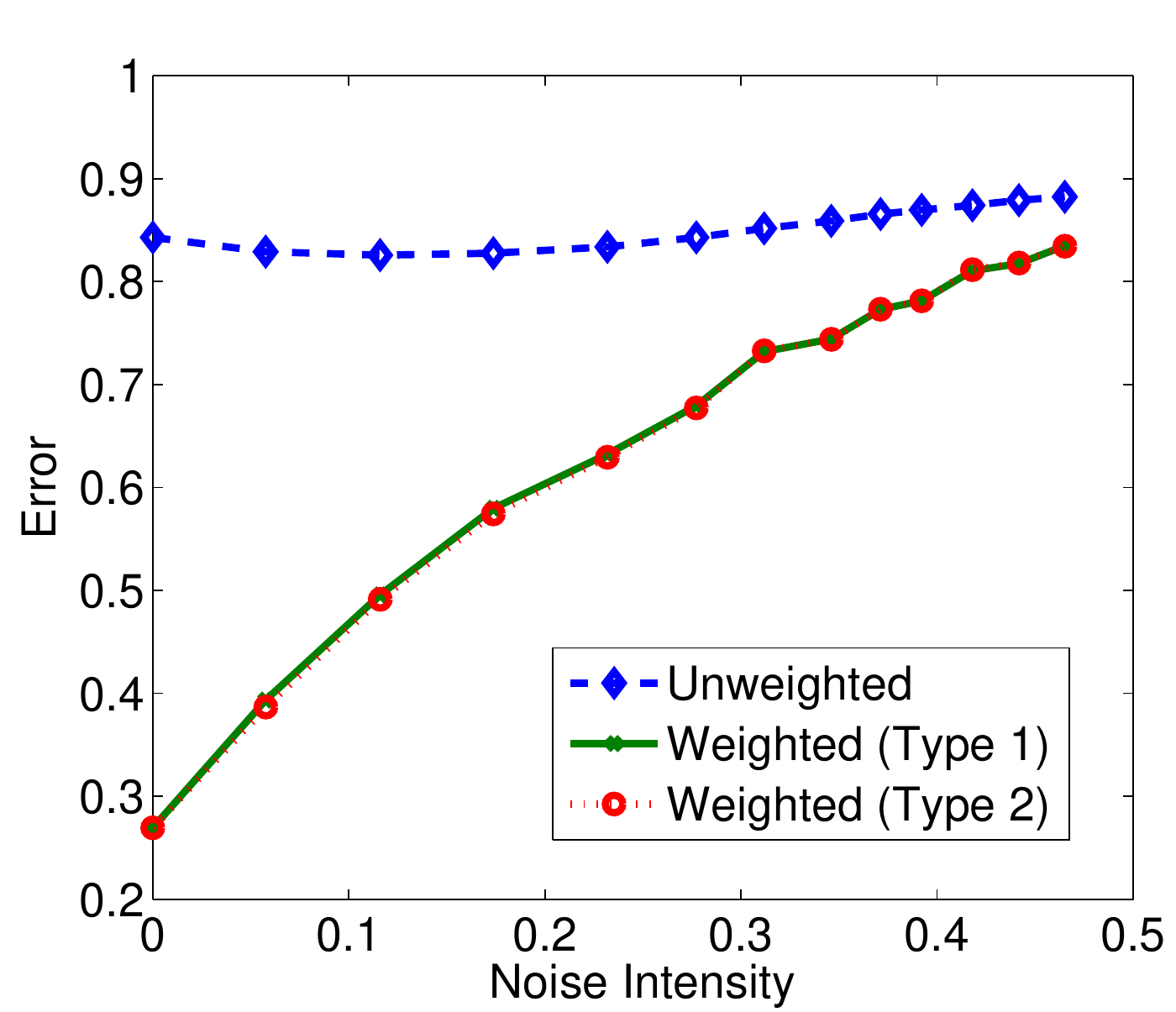}}
\subfigure[$p=10\%$]{\includegraphics[width=0.15\textwidth]{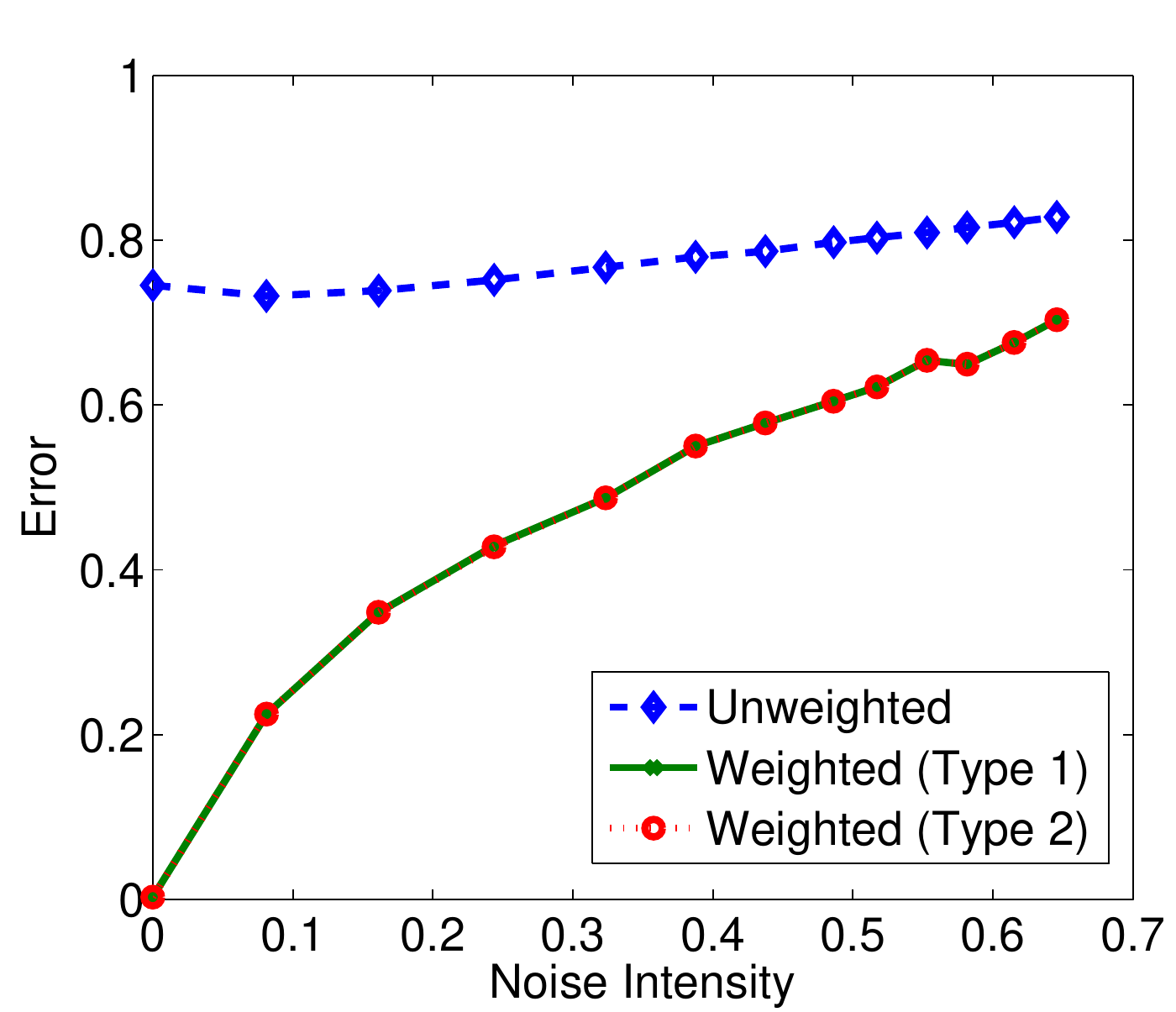}}
\subfigure[$p=20\%$]{\includegraphics[width=0.15\textwidth]{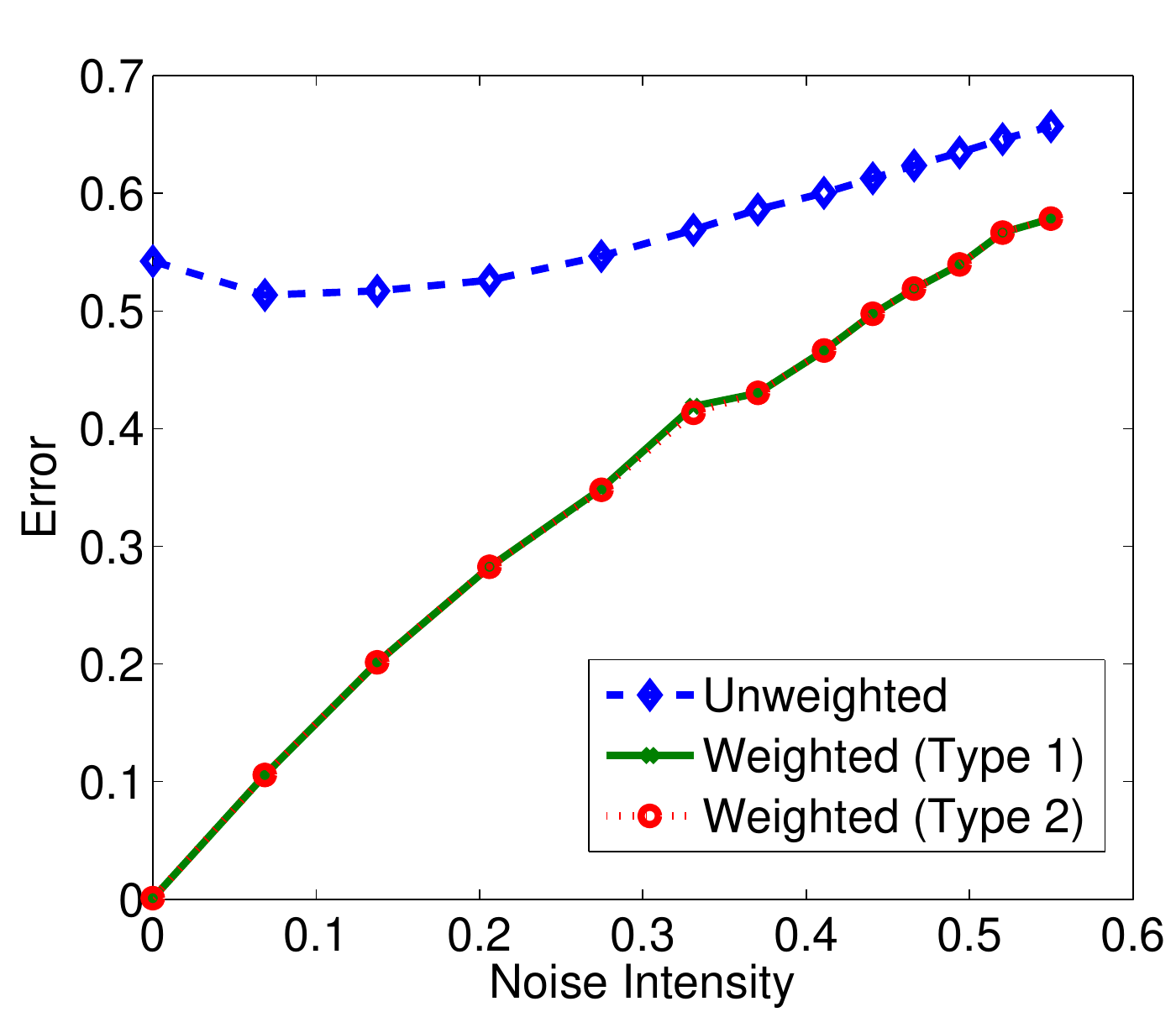}}
\end{center}
   \caption{Experiments on noisy matrix completion.}
\label{fig:noisyMC}
\end{figure}

The experimental results show that when the underlying low-rank matrix is coherent,
the standard nuclear norm minimization method fails even on the noise-free data with a large number of samples, i.e., $p = 20\%$ entries observed.
In comparison, our weighted method strictly succeeds on the noise-free data with $p = 10\%$ or $20\%$,
and our method achieves high accuracy even under heavy noise.

\section{Extension to the Robust Principal Component Analysis} \label{sec:rpca}

Our proposed row/column weighting method can also be applied to low-rank matrix recovery problems
other than matrix completion.
The robust principal component analysis (RPCA) is a well-known low-rank plus sparse matrix recovery model,
but it requires the matrix coherence to be small in order to recover the underlying low-rank and sparse matrices.
In this section we apply our method to RPCA and call the obtain method the weighted RPCA.
We show that highly coherent low-rank matrices can also be recovered by the weighted RPCA.

\subsection{The Standard RPCA Method}

In some real-world applications the observation $\D \in \RB^{n_1 \times n_2}$ is the superposition of a low-rank matrix $\LL_0$ and a sparse matrix $\S_0$. It is useful to recover the low-rank matrix and the sparse matrix from the observation.
For example, in the video surveillance problem,
by vectorizing each frame of a video surveillance sequence and stacking the obtained vectors,
the obtained matrix is the sum of the low-rank background and the sparse foreground.

The robust principal component analysis (RPCA) \cite{candes2011robust} is perhaps the most effective tool for such a matrix recovery task.
The RPCA model is defined as follows:
\begin{eqnarray} \label{eq:weightedRPCA}
\min_{\LL, \S} \; \big\| \S \big\|_1 + \lambda \| \LL \|_*;
\qquad \st \;  \LL + \S = \D.
\end{eqnarray}%
Let $\tau = \max \big\{ \mu (\LL_0) , \nu (\LL_0) ,$ $ \frac{n_1 n_2}{k} $ $\| \U_{\LL_0 , k} \V_{\LL_0 , k}^T \|_{\infty}^2 \big\} $
and $k = \rk (\LL_0)$.
Suppose  $n_1 \geq n_2$.
It was shown in \cite{candes2011robust} that RPCA exactly recovers $\LL_0$ and $\S_0$ provided that
$k \leq c_l (n_1 + n_2) \tau^{-1} \log^{-2} n_1 $
and $\|\S_0 \|_0 \leq c_s n_1 n_2$
for some constants $c_l$ and $c_s$.
Therefore, the exact recovery is guaranteed only when the matrix coherence parameter $\tau$ is small.

\subsection{Weighted RPCA}

When the underlying low-rank matrix $\LL_0$ is coherent,
the standard RPCA method can easily fail.
To remedy this problem, we propose the use of weighted RPCA to recover coherent low-rank matrix.

Let $\R$ and $\C$ be some diagonal matrices.
We can weight $\D$ by
\vspace{-1mm}
\begin{small}
\begin{equation}
\underbrace{\R \D \C}_{\hat{\D}} \; = \; \underbrace{\R \LL_0 \C}_{\hat{\LL}_0} + \underbrace{\R \S_0 \C }_{\hat{\S}_0} . \nonumber
\vspace{-1mm}
\end{equation}
\end{small}%
It is obvious that $\rk (\hat{\LL}_0) = \rk (\LL_0) = k$ and $\|\hat{\S}_0\|_0 = \|{\S}_0\|_0$;
thus $\hat{\D}$ is the sum of the low-rank matrix $\hat{\LL}_0$ and the sparse matrix $\hat{\S}_0$.
From the analysis of \cite{candes2011robust} we know that the matrix recovery performance can be improved if
the matrix coherence parameters of $\hat{\LL}_0$ is lower than those of ${\LL}_0$.
We therefore propose to use Algorithm~\ref{alg:weighing} to compute $\R$ and $\C$ and do row and column weighting before performing RPCA.
Our proposal is as follows:
\vspace{-1mm}
\begin{enumerate}
\item
    Compute $\R$ and $\C$ using Algorithm~\ref{alg:weighing} which takes $\D$ and $k$ as inputs.
\vspace{-1mm}
\item
    Take $\hat{\D} = \R \D \C$ instead of $\D$ as the input of RPCA and obtain the solution $\hat{\LL}^\star$, $\hat{\S}^\star$.
\vspace{-1mm}
\item
    Output $\LL^\star = \R^{-1} \hat{\LL}^\star \C^{-1}$ and $\S^\star = \R^{-1} \hat{\S}^\star \C^{-1}$.
\vspace{-1mm}
\end{enumerate}

The experiments in Section~\ref{sec:RPCA:experiments:leverage_scores}
shows that the performance of row weighting deteriorates as $\|\S_0\|_0$ or $\|\S_0\|_{\infty}$ increases.
In the same spirit with the weighting---completion algorithm in Section~\ref{sec:weighting_completion},
we propose to use the recovered matrix $\LL^\star$ instead of $\D$ to compute the weight matrices $\R$ and $\C$.
In these experiments we use two methods to compute the weight matrices $\R$ and $\C$:
\begin{itemize}
\vspace{-1mm}
\item
    {\bf Type 1.}
    Use $\D$ and $k$ as the input of Algorithm~\ref{alg:weighing} to compute $\R$ and $\C$ and perform the weighted RPCA to recover $\LL^\star$.
\vspace{-1mm}
\item
    {\bf Type 2.}
    Use Type 1 to recover $\LL^\star$, and then use $\LL^\star$ and $k$ as the input of Algorithm~\ref{alg:weighing} to obtain $\hat\R$ and $\hat\C$.
    Finally run the weighted RPCA once more using the new weight matrices $\hat\R$ and $\hat\C$.
\end{itemize}

\begin{figure}[!ht]
\subfigtopskip = 0pt
\begin{center}
\centering
\subfigure[Matrix coherence versus $t$]{\includegraphics[width=0.23\textwidth]{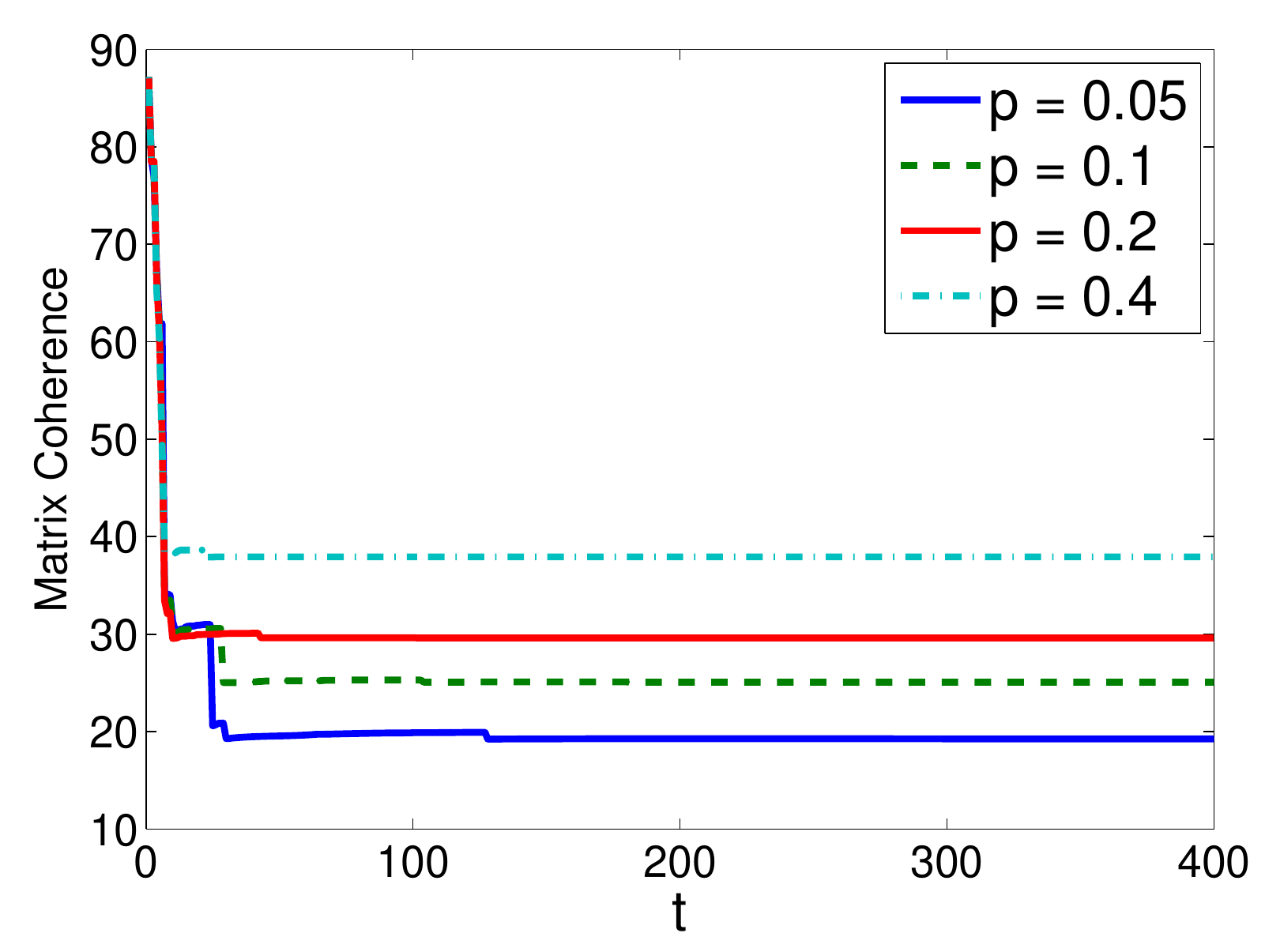}}
\subfigure[The $\ell_1$ loss function value versus $t$]{\includegraphics[width=0.23\textwidth]{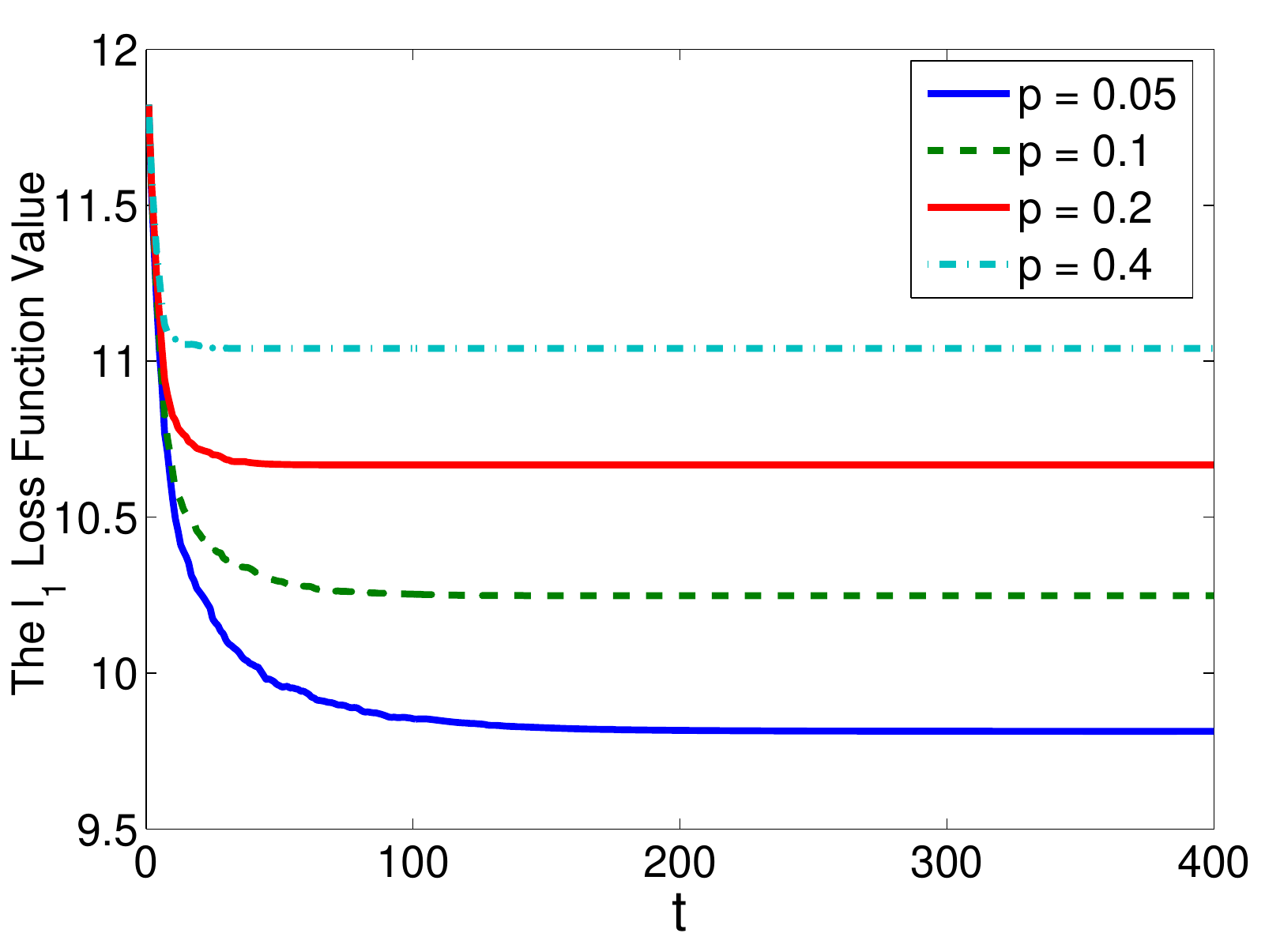}}
\end{center}
   \caption{The effects of row weighting under varying $\|\S_0\|_0$.}
\label{fig:LevDensityRPCA}
\end{figure}

\begin{figure}[!ht]
\subfigtopskip = 0pt
\begin{center}
\centering
\subfigure[$t$ versus matrix coherence]{\includegraphics[width=0.23\textwidth]{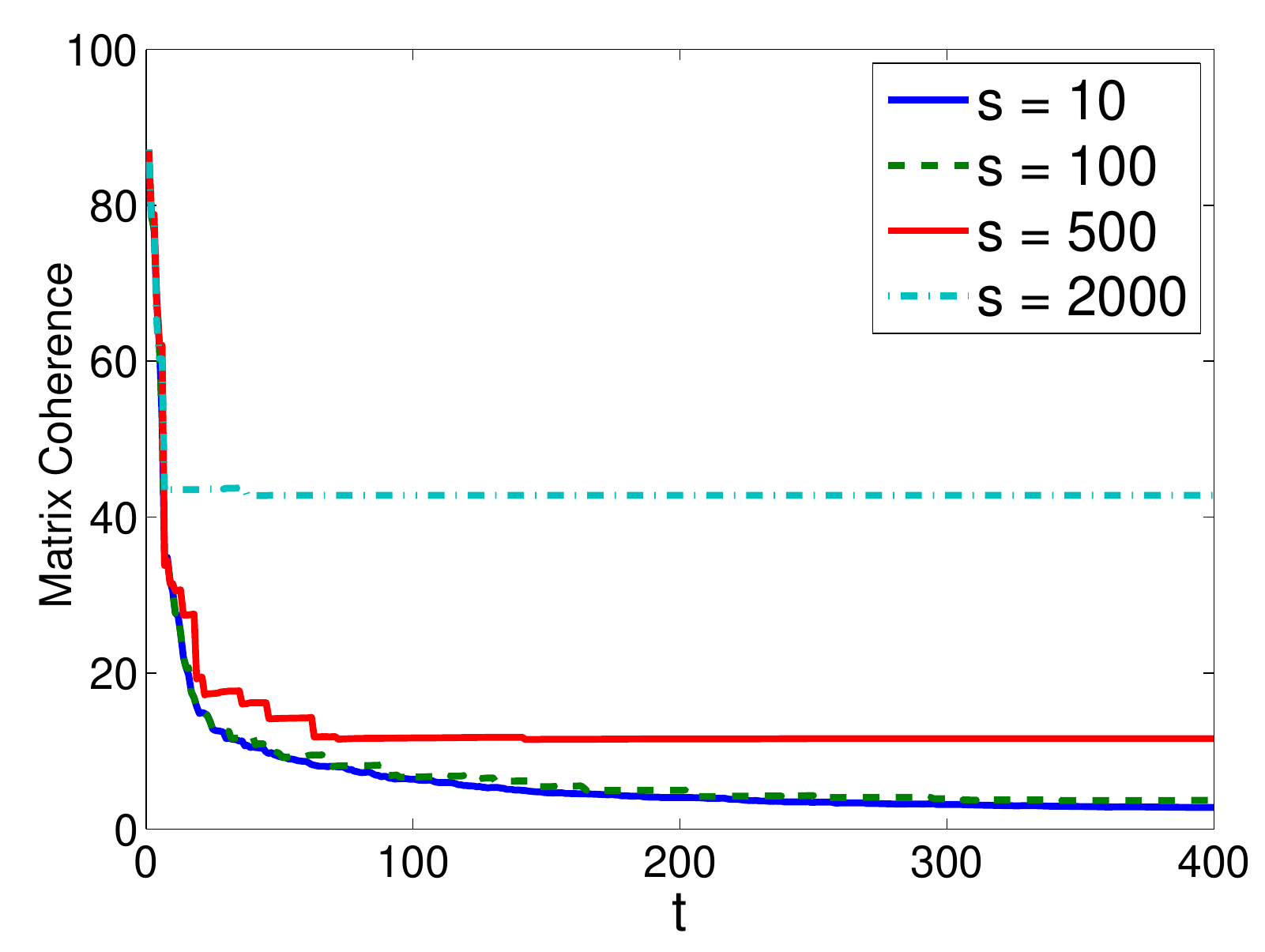}}
\subfigure[$t$ versus the $\ell_1$ loss function value]{\includegraphics[width=0.23\textwidth]{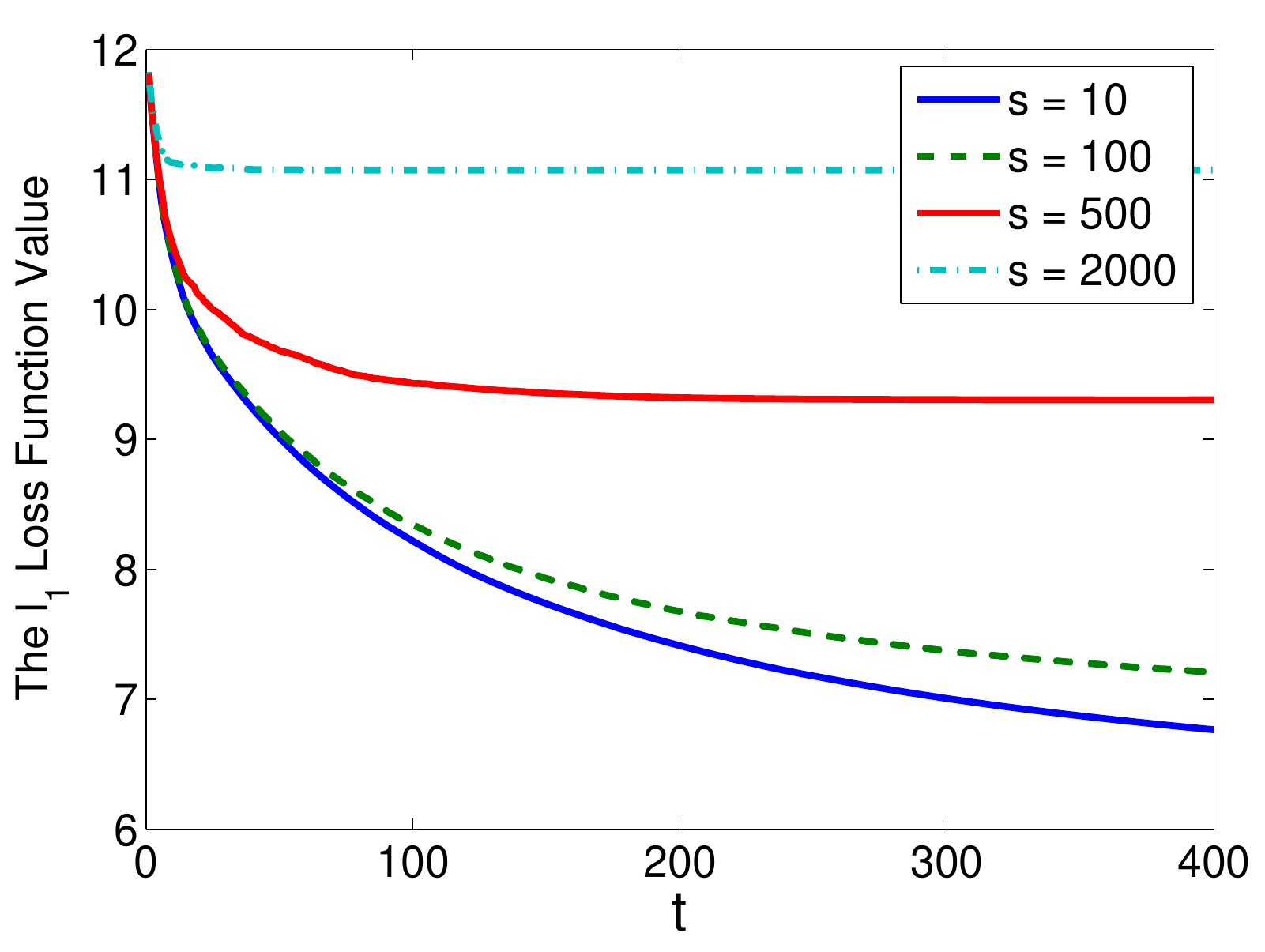}}
\end{center}
   \caption{The effects of row weighting under varying $\|\S_0\|_{\infty}$.}
\label{fig:MagnitudeDensityRPCA}
\end{figure}

\subsection{Effects on Leverage Scores} \label{sec:RPCA:experiments:leverage_scores}

We set $n_1 = 2,000$, $n_2 = 1,000$, $k = 20$ and use the data model defined in Section~\ref{sec:datasets} to generate the low-rank matrix $\LL_0$.
The sparse matrix $\S_0$ is generated independently by uniformly setting entries to be $s$ with probability $p/2$,
$-s$ with probability $p/2$, and zero with probability $1-p$.
That is $\|\S\|_0 = p n_1 n_2$ and $\|\S\|_\infty = s$.
We use $\D = \LL_0 + \S_0$ as the observation.

In the first set of experiments, we illustrate the effect of Algorithm~\ref{alg:weighing},
which takes $\D$ and $k$ as inputs, with varying $\S_0$.
We fix $s = 1,000$ and vary $p$, and plot the results in Figure~\ref{fig:LevDensityRPCA}.
We then fix $p = 0.1$ and vary $s$, and plot the results in Figure~\ref{fig:MagnitudeDensityRPCA}.
The results clearly indicate that under the RPCA setting,
our row weighting algorithm makes the leverage scores more uniform and the matrix coherence much lower.

\subsection{Matrix Recovery Accuracy}

In the second set of experiments,
we compare the weighted and unweighted RPCA in terms of matrix recovery accuracy.
The relative error defined in (\ref{eq:def_error}) is used
to evaluate the recovery accuracy.

We generate the synthetic data $\LL_0$ in the same way as in the previous subsection.
As for the sparse matrix $\S_0$, we first fix $s=1,000$ and vary $p$,
and report the relative error in Figure \ref{fig:ErrorRPCA:p}.
We then fix $p = 0.2$ and vary $s$,
and report the relative error in Figure \ref{fig:ErrorRPCA:s}.

Since the low-rank matrix $\LL_0$ is highly coherent,
the standard RPCA fails even if $\|\S_0\|_0$ and $\| \S \|_{\infty}$ are both very small.
Since we use $\D = \LL_0 + \S_0$ to compute the weight matrices $\R$ and $\C$,
the error in the leverage score estimation should be lower if $\D$ is closer to $\LL_0$.
The Type 1 weighted RPCA achieves much better performance when
$\|\S_0\|_0$ and $\| \S \|_{\infty}$ are reasonably small,
which is in accordance with our expectation.
We can also see that the Type 2 weighted RPCA achieves the highest accuracy.
The computational cost of the Type 2 method is only twice as much as the Type 1 method,
but the accuracy is much higher than that of Type 1.

\begin{figure}[!ht]
\subfigtopskip = 0pt
\begin{center}
\centering
\subfigure[Fix $s=1,000$, vary $p$.]{\includegraphics[width=0.23\textwidth]{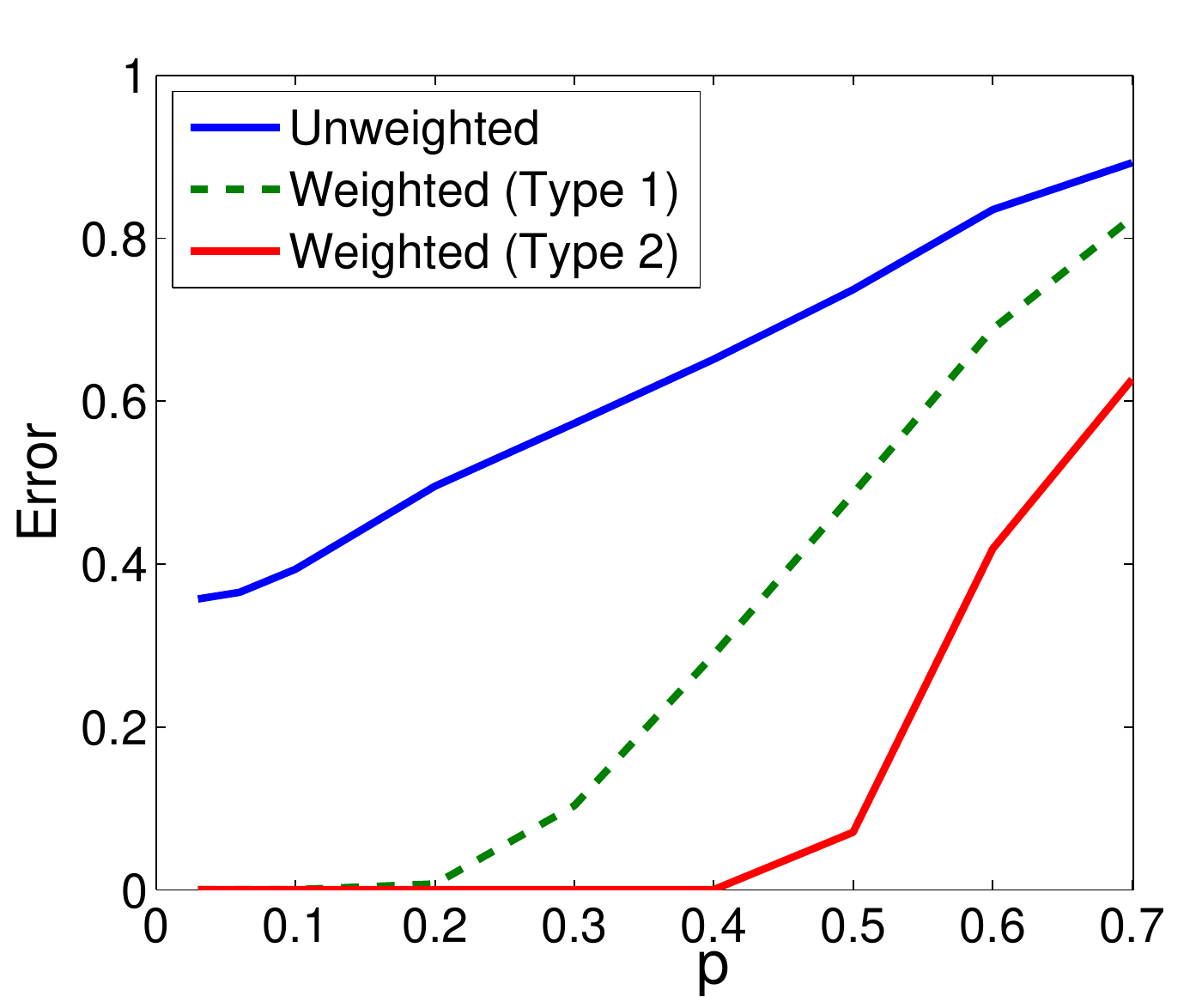}\label{fig:ErrorRPCA:p}}
\subfigure[Fix $p=0.2$, vary $s$.]{\includegraphics[width=0.23\textwidth]{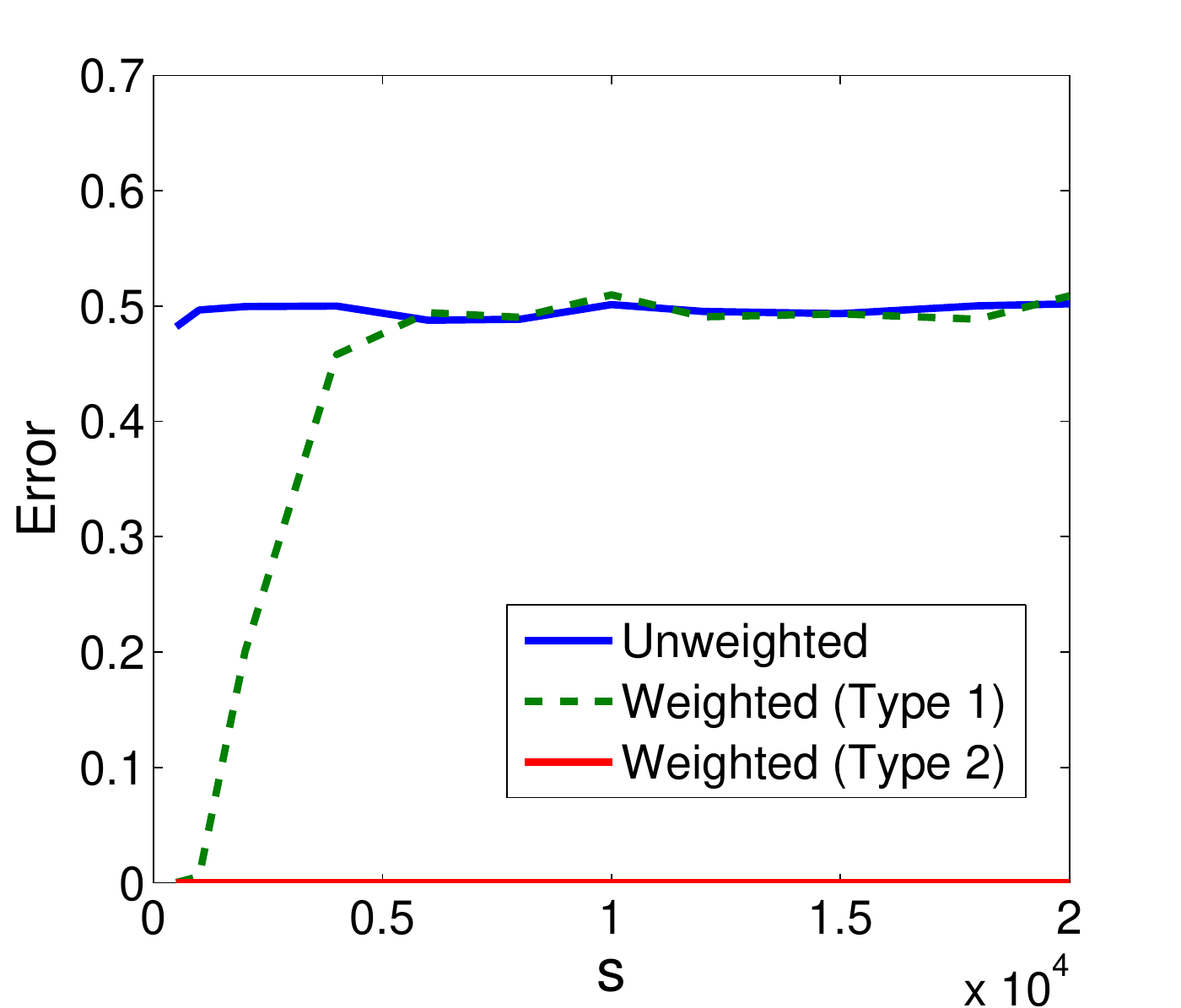}\label{fig:ErrorRPCA:s}}
\end{center}
   \caption{The matrix recovery errors.}
\label{fig:ErrorRPCA}
\end{figure}

\section{Conclusions and Future Work} \label{sec:conclusion}

In this paper  we have proposed a row/column weighting method for adjusting the leverage scores
such that low-rank matrix recovery can be better performed.
We have established an optimization model that describes the discrepancy between the leverage scores and desired leverage scores.
We have also devised a coordinate descent algorithm for solving the model without knowing the true leverage scores.
We have proved that this algorithm does not increase the objective function value and that it decreases the objective function value under certain conditions.
We have applied our matrix weighting method to the matrix completion problem such that
coherent low-rank matrices can be completed more accurately via the weighted nuclear norm minimization formulation.
Moreover, we have  applied our matrix weighting method to the robust principal component analysis  problem,
so that highly coherent low-rank matrices can be recovered from noisy observations more accurately.

We should point out that currently our proposed method applies only to the uniform sampling model;
that is, all matrix entries are observed with the same probability.
Our row weighting method critically relies on the quality of leverage score estimation.
For the non-uniform sampling model, it is not clear how to accurately estimate the leverage scores from partial observation.
For the non-uniform sampling model, however, if we have a method that can accurately estimate the leverage scores,
the same coordinate descent algorithm can be used to find the weight matrix $\R$ by solving (\ref{eq:def_lq_loss}).
Considering that most real-world applications obey non-uniform sampling settings such as the power law distribution,
it is very useful to find a way to estimating leverage scores from non-uniformly sampled entries,
with which our row weighting approach can demonstrate its power in the real-world applications.

\bibliographystyle{abbrv}
\bibliography{adjusting}

\newpage
\appendix

\section{ADMM for Weighted Nuclear Norm Minimization} \label{sec:admm}

The regularized weighted nuclear norm minimization formulation (\ref{eq:weightedMC_noise}) can be equivalently converted to (\ref{eq:weightedMC_noise_slack})
by adding the slack variable $\X$:
\begin{eqnarray} \label{eq:weightedMC_noise_slack}
\min_{\LL, \X} & \frac{1}{2} \big\| \PM_{\Omega} (\M - \LL) \big\|_F^2 + \lambda \| \X \|_*, \nonumber \\
\st & \X = \R \LL \C.
\end{eqnarray}
The augmented Lagrange function of this problem is
\begin{footnotesize}
\begin{eqnarray*}
f_{\gamma} (\LL, \X, \Y)
& = &  \frac{1}{2} \big\| \PM_{\Omega} (\M - \LL) \big\|_F^2 + \lambda \| \X \|_* \\
& &       + \langle \Y , \, \R \LL \C - \X \rangle + \frac{\gamma}{2} \big\| \R \LL \C - \X \big\|_F^2 \nonumber  \\
& = &  \frac{1}{2} \big\| \PM_{\Omega} (\M - \LL) \big\|_F^2 + \lambda \| \X \|_* \\
&&        + \frac{\gamma}{2} \big\| {\gamma}^{-1} \Y + \R \LL \C - \X \big\|_F^2 - \frac{\gamma}{2} \big\| {\gamma}^{-1}\Y \big\|_F^2. \nonumber
\end{eqnarray*}
\end{footnotesize}
We alternately minimize $f_\gamma$ w.r.t.\ $\LL$ and $\X$ and maximize $f_{\gamma}$ w.r.t.\ $\Y$.

The terms in $f_\gamma$ containing $\PM_{\Omega} (\LL)$ is
\begin{small}
\begin{eqnarray}
f_{\gamma} \big(\PM_{\Omega} (\LL) \big)
= \frac{1}{2} \big\| \PM_{\Omega} (\M - \LL) \big\|_F^2
    + \frac{\gamma}{2} \Big\| \PM_{\Omega} \Big(\frac{\Y}{\gamma}  - \X + \R \LL \C \Big) \Big\|_F^2 , \nonumber
\end{eqnarray}
\end{small}
and the minimizer is
\begin{small}
\begin{equation} \label{eq:alg:update_L_in_omega}
L_{i j}^\star \; = \;
\big(1+ \gamma R_{ii}^2 C_{jj}^2\big)^{-1} \big(M_{ij} + \gamma R_{ii} C_{jj} X_{i j} - R_{ii} C_{jj} Y_{i j} \big)
\end{equation}
\end{small}
for all $(i,j) \in \Omega$.

Let $\PM_{\Omega}^{\perp} (\LL)$ be the $n_1 \times n_2$ matrix whose the $(i,j)$-th entry is $L_{i j}$ if $(i,j) \not\in \Omega$ and is zero otherwise.
The terms in $f_\gamma$ containing $\PM_{\Omega}^{\perp} (\LL)$ is
\begin{eqnarray}
f_{\gamma} \big(\PM_{\Omega}^{\perp} (\LL) \big)
& = & \frac{\gamma}{2} \big\| \PM_{\Omega}^{\perp} \big({\gamma}^{-1}\Y  - \X + \R \LL \C \big) \big\|_F^2 , \nonumber
\end{eqnarray}
and the minimizer is
\begin{equation} \label{eq:alg:update_L_notin_omega}
L_{i j}^\star \; = \; R_{i i}^{-1} C_{j j}^{-1} \big( X_{i j} - \gamma^{-1} Y_{i j} \big)
\end{equation}
for all $(i,j) \not\in \Omega$.

The terms in $f_\gamma$ containing $\X$ is
\begin{eqnarray}
f_{\gamma} (\X)
= \frac{\gamma}{2} \big\| {\gamma}^{-1}\Y + \R \LL \C - \X \big\|_F^2  + \lambda \|\X\|_*, \nonumber
\end{eqnarray}
and the minimizer is
\begin{equation} \label{eq:alg:update_X}
\X^\star \; = \;
    \U \big[ \Si - \gamma^{-1} \lambda  \I \big]_+ \V^T,
\end{equation}
where the $(i,j)$ of $[\M]_+$ is $\max\{A_{i j} , 0\}$ and
\[
[\U, \Si, \V] \; \longleftarrow \; \textrm{SVD} \big( {\gamma}^{-1}\Y + \R \LL \C \big).
\]

We can update $\LL$ by (\ref{eq:alg:update_L_in_omega}) and (\ref{eq:alg:update_L_notin_omega}),
update $\X$ by (\ref{eq:alg:update_X}),
and update the dual variable $\Y$ by gradient ascent: $\Y \longleftarrow \Y + \gamma (\R \LL \C - \X)$.
The whole procedure is described in Algorithm~\ref{alg:admm}.

\begin{figure*}[!ht]
\subfigtopskip = 0pt
\begin{center}
\centering
\subfigure[the $\ell_1$ loss function value]{\includegraphics[width=0.32\textwidth]{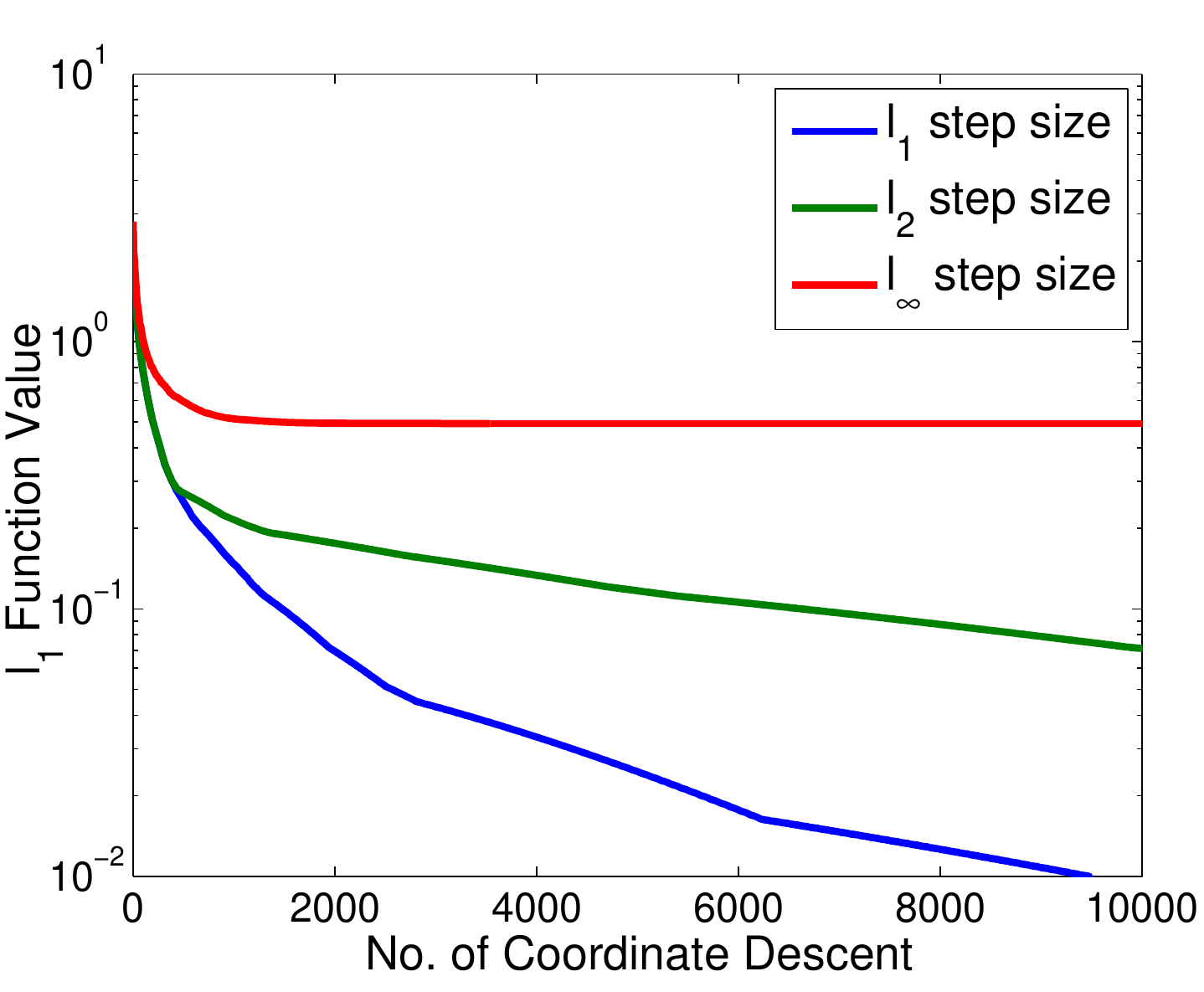}}
\subfigure[the $\ell_2$ loss function value]{\includegraphics[width=0.32\textwidth]{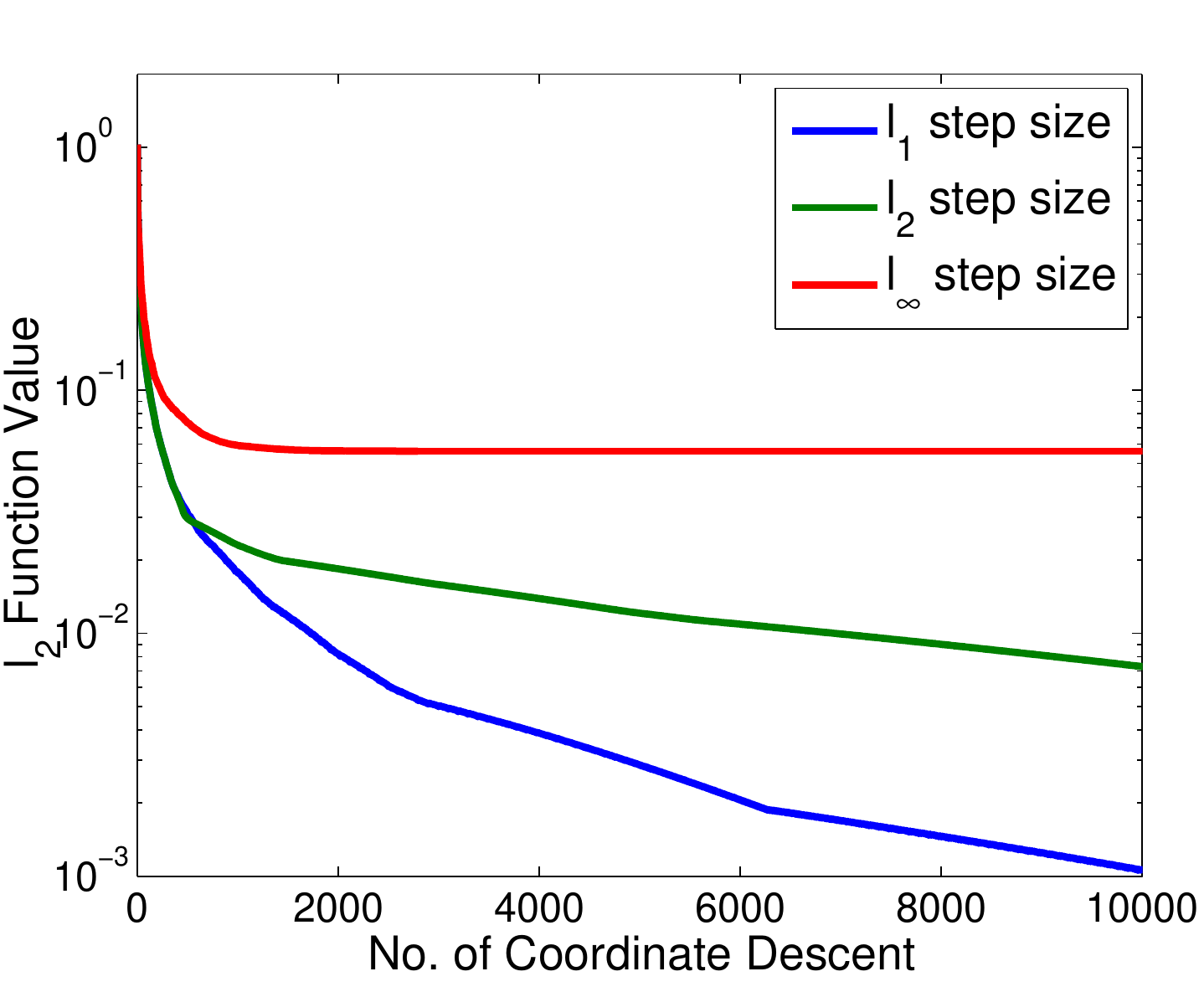}}
\subfigure[the $\ell_\infty$ loss function value]{\includegraphics[width=0.32\textwidth]{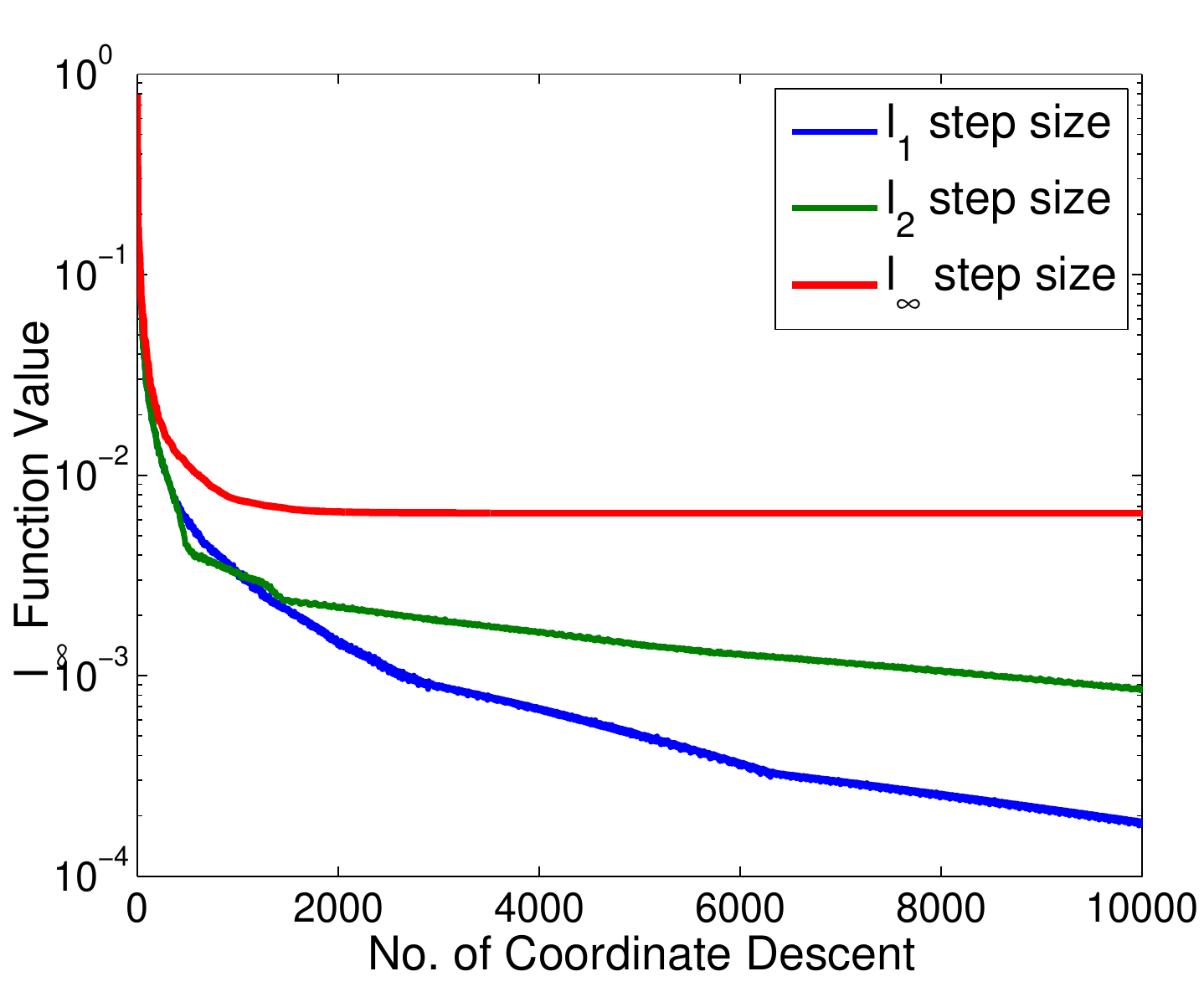}}
\end{center}
   \caption{The comparisons among the loss functions.}
\label{fig:lossfunctions}
\end{figure*}

\section{Comparing among Different Loss Functions} \label{sec:compare_lq}

We first give an example to show the advantage of optimizing the $\ell_1$ hinge loss function.

In this example, if we optimize the $\ell_\infty$ hinge loss function by coordinate descent,
scaling any row by any factor in the region $(0, 1)$ does not decrease the $\ell_\infty$ hinge loss function value.
That is, the coordinate descent algorithm stops in this configuration.

In comparison, in this example, if we optimize the $\ell_1$ hinge loss function by coordinate descent,
the $1$st or $2$nd row will be scaled by a factor $\gamma\in (0,1)$,
and the $\ell_1$ hinge loss function value will decrease in certain conditions (see Theorem~\ref{thm:approx_cd}).
In this way, the coordinate descent proceeds without being stuck in the configuration in the example.

\begin{example} \label{example:l_infty}
Let $\mu_1, \cdots , \mu_{n_1}$ be the row leverage scores of a rank $k$ matrix $\M\in \RB^{n_1\times n_2}$,
and the ideal leverage scores are $\mu_1^\star = \cdots = \mu_{n-1}^\star = \frac{k}{n_1}$.
When $\mu_1 = \mu_2 > \frac{k}{n_1} \geq \mu_3 , \cdots , \mu_{n-1}$,
weighting any single row by any factor $\gamma \in (0, 1)$ cannot decrease the $\ell_\infty$ hinge loss function value.
\end{example}

\begin{proof}
If we scale the $1$st row by any factor $\gamma \in (0, 1)$,
the leverage score of the $2$nd row will be nondecreasing (by Lemma~\ref{lem:weighting}).
Thus the the $\ell_\infty$ hinge loss function value does not decrease.

The same happens if we scale the $2$nd row.

If we scale the $i$-th row for $i > 2$,
the leverage scores of the $1$st and $2$nd rows will be nondecreasing (by Lemma~\ref{lem:weighting}).
Thus the the $\ell_\infty$ hinge loss function value does not decrease.
\end{proof}

Here we use a toy dataset to compare among the $\ell_1$, $\ell_2$, and $\ell_\infty$ hinge loss functions.
We generate a synthetic data matrix $\M$ according to Section~\ref{sec:datasets} by setting $n_1 = n_2 = 100$ and $k = 5$.
We assume the exact leverage scores are available to us
and use coordinate descent algorithm to optimize the loss functions.
We use three kinds of step sizes: the $\ell_1$, $\ell_2$, and $\ell_\infty$ step sizes;
``the $\ell_q$ step size'' means the step size computed by line search that decrease the $\ell_q$ hinge loss function value.
We plot in Figure~\ref{fig:lossfunctions} the loss function values versus the number of coordinate descents.
The $\ell_1$ step size always leads to the fastest convergence;
the $\ell_1$ step size does not only decrease the $\ell_1$ function value fast, but also decrease the $\ell_2$ and $\ell_\infty$ very fast.
The $\ell_2$ step size also works well, but it is much smaller and results in slow convergence.
The experiment shows that directly minimizing the $\ell_\infty$ hinge loss function is problematic because it can get stuck in some bad local minimum configurations without making progress,
which corroborates the analysis in Example~\ref{example:l_infty}.

\section{Proof of the Theorems}

\section{Key Lemmas}

\begin{lemma} \label{lem:decrement}
Let $i$ be the index selected in the $t$-th step of the coordinate descent algorithm.
Suppose the following equality holds:
\[
\mu_i (\R^{(t-1)} \M)
\;>\;\mu_i (\R^{(t)} \M)
\; \geq \; \mu_i^\star
\]%
Let $\Delta \mu_j$ be defined in (\ref{eq:l1_loss_decrement1}) and
the set $\JM$ be defined in (\ref{eq:l1_loss_decrement_set}).
Then the decrement in the $\ell_1$ hinge loss function is
\begin{small}
\begin{align*}
&L_{\M,1} \big(\R^{(t-1)}\big) - L_{\M,1} \big(\R^{(t)}\big) \nonumber \\
& = \;
\sum_{ j \in \JM }
\min \bigg\{  \Delta \mu_j , \; \mu_j^\star - \mu_j \big(\R^{(t-1)}\M\big) \bigg\}
\:\geq \: 0,
\end{align*}
\end{small}
\end{lemma}

\begin{proof}
Let $i$ be the selected index in the $t$-th step.
For convenience, we write the definition of the $\ell_1$ hinge loss function here:
\begin{align*}
L_{\M, 1} \big(\R^{(t)}\big) \; := \;  \sum_{l=1}^{n_1} \max \Big\{ \mu_l \big(\R^{(t)} \M \big) - \mu_l^\star , \; 0 \Big\}.
\end{align*}
Obviously, the decrease in the $i$-th leverage score leads to the decrease in the loss function value,
and the increase in the $j$-th ($j\neq i$) leverage scores may increase the loss function value.

The decrease in the $i$-th leverage score contributes to a decrease in the $\ell_1$ hinge loss function:
\[
\delta_1 := \mu_i \big(\R^{(t-1)} \M \big) -\mu_i \big(\R^{(t)} \M \big) > 0.
\]
From Lemma~\ref{lem:weighting} we know that $\mu_j \big(\R^{(t)} \M \big) \geq \mu_j \big(\R^{(t-1)} \M \big)$ for all $j\neq i$.
The increase in the leverage scores $\{\mu_j\}_{j\neq i}$ contributes to a increase in the $\ell_1$ hinge loss function:
\begin{small}
\begin{eqnarray*}
\delta_2
& := & \sum_{j\in \JM_1 } \Delta \mu_j
     +  \sum_{j\in \JM_2 }  \max \Big\{ \mu_j (\R^{(t-1)} \M) + \Delta \mu_j - \mu_j^\star , 0  \Big\} \geq 0
\end{eqnarray*}
\end{small}%
where $\JM_1 = \{ j\neq i \,|\, \mu_j (\R^{(t-1)} \M ) \geq \mu_j^\star \}$
and $\JM_2 = \{ j\neq i \,|\, \mu_j (\R^{(t-1)} \M ) < \mu_j^\star  \}$.
Thus the decrease the the $\ell_1$ hinge loss function is
\begin{small}
\begin{align*}
&L_{\M,1} \big(\R^{(t-1)}\big) - L_{\M,1} \big(\R^{(t)}\big) \nonumber \\
& = \; \delta_1 - \delta_2 \nonumber \\
& = \sum_{j\in \JM_1 } \Delta \mu_j + \sum_{j\in \JM_2 } \Delta \mu_j - \delta_2\nonumber \\
& = \; \sum_{j\in \JM_2 } \bigg[ \Delta \mu_j -
        \max \Big\{ \mu_j (\R^{(t-1)} \M) + \Delta \mu_j - \mu_j^\star , 0  \Big\} \bigg] \nonumber \\
& = \; \sum_{ j \in \JM_2 } \min \Big\{  \Delta \mu_j , \; \mu_j^\star - \mu_j \big(\R^{(t-1)}\M\big) \Big\} \nonumber\\
& = \; \sum_{ j \in \JM } \min \Big\{  \Delta \mu_j , \; \mu_j^\star - \mu_j \big(\R^{(t-1)}\M\big) \Big\}.
\end{align*}
\end{small}%
Here the second equality follows from the definition of $\Delta \mu_j$ that $\delta_1 = \sum_{j\neq i} \Delta \mu_j$.
The last equality follows from that $\Delta \mu_j = 0$ if $\mu_{i j} \big(\R^{(t-1)} \M\big) = 0$.
\end{proof}

\begin{lemma} \label{lem:monotone}
Let $i$ be the index selected in the $t$-th step of the coordinate descent algorithm,
and the $i$-th row of $\R^{(t-1)} \M$ is scaled by $\sqrt{1-\gamma} \in (0, 1)$ to obtain  $\R^{(t)} \M$.
Suppose $\mu_i (\R^{(t-1)} \M) > \mu_i^\star$.
Then we have
\begin{enumerate}
\item
    The leverage score $\mu_i (\R^{(t)} \M)$ decrease as $\gamma$ increases.
\item
    When $\gamma \leq \frac{1- \mu_i^\star / \mu_i (\R^{(t-1)} \M)}{1 - \mu_i^\star}$,
    or equivalently, $\mu_i (\R^{(t)} \M) \geq \mu_i^\star$,
    the decrement in the $\ell_1$ hinge loss function (defined in Lemma~\ref{lem:decrement}) increases as $\gamma$ increases.
\item
    When $\gamma > \frac{1- \mu_i^\star / \mu_i (\R^{(t-1)} \M)}{1 - \mu_i^\star}$,
    or equivalently, $\mu_i (\R^{(t)} \M) < \mu_i^\star$,
    the decrement in the $\ell_1$ hinge loss function does not increase as $\gamma$ increases.
\end{enumerate}
\end{lemma}

\begin{proof}
The first property follows directly from Lemma~\ref{lem:weighting}.

Second, from the definition of $\Delta \mu_j$ in (\ref{eq:l1_loss_decrement1}),
we know that $\Delta \mu_j$ increases as $\mu_i (\R^{(t)} \M)$ decreases.
When $\mu_i \big( \R^{(t)} \M \big) \geq \mu_i^\star$ holds,
Lemma~\ref{lem:decrement} shows that the decrement in the $\ell_1$ hinge loss function value increases as $\Delta \mu_j$ increases.
Hence the decrement in the $\ell_1$ hinge loss function value increases as $\gamma$ increases.

Third, when $\mu_i (\R^{(t)} \M) < \mu_i^\star$, as $\gamma$ increase,
the term $\max$$ \{ \mu_i (\R^{(t)} \M) - \mu_i^\star , 0 \}$ remains constant,
and for all $j$ the terms $\max$$ \{ \mu_j (\R^{(t)} \M) - \mu_i^\star , 0 \}$ does not decrease.
As $\gamma$ increase, the $\ell_1$ hinge loss function $L_{\M, 1} (\R^{(t)}$ does not decrease
and thus the decrement in the $\ell_1$ hinge loss function does not increase.
\end{proof}

\subsection{Proof of Theorem~\ref{thm:steepest_descent}}
The theorem follows directly from Lemma~\ref{lem:monotone}.

\subsection{Proof of Theorems~\ref{thm:perturbation}}

%
%
%

It was shown in Lemma C.1 of
\cite{jain2013low} shows that with probability at least $1-n^{-3}$ the following inequality holds:
\[
\big\| \PM_{\M}^{\perp} \U_{\widetilde{\M},k} \big\|_2
\; \leq \;
C' \frac{\sigma_1 (\M)}{\sigma_k (\M)} \sqrt{\frac{n k^2}{|\Omega|}},
\]
where the projection operator $\PM_{\M}^{\perp} = \I - \M \M^\dag$.
Thus we have the angular between the ranges of $\M$ and $\widetilde{\M}$ is
\[
\sin\angle \big( \M, \: \widetilde{\M}_k \big)
\; \leq \;
\big\| \PM_{\M}^{\perp} \U_{\widetilde{\M},k} \big\|_2
\; \leq \;
\frac{1}{2 \rho}
\]
for a large enough $C$ defined in the theorem.
%

We have that for all $i, j \in [n] $, the leverage score and the cross leverage scores satisfy
\begin{small}
\begin{eqnarray*}
\left.
  \begin{array}{c}
    \big| \ell_i (\M) - \ell_i (\widetilde{\M}_k) \big| \\
    \big| \ell_{i j} (\M) - \ell_{i j} (\widetilde{\M}_k) \big| \\
  \end{array}
\right\}
& \leq &
\big\|   \U_{\M,k}  \U_{\M,k}^T - \U_{\widetilde{\M},k}  \U_{\widetilde{\M},k}^T \big\|_{\max} \\
& \leq &
\big\|   \U_{\M,k}  \U_{\M,k}^T - \U_{\widetilde{\M},k}  \U_{\widetilde{\M},k}^T \big\|_{2} \\
& = &
\sin \angle \big( \M, \: \widetilde{\M}_k \big) \\
& \leq &
\frac{1}{2 \rho},
\end{eqnarray*}
\end{small}
from which the theorem follows.

\subsection{Proof of Theorems~\ref{thm:mu_medium}}

In this section we use the simplified notation defined in Table~\ref{tab:notation}.

\begin{table}[!ht]\setlength{\tabcolsep}{0.3pt}
\caption{Notation} \label{tab:notation}
\vspace{2mm}
\centering
\begin{footnotesize}
\begin{tabular}{c | c}
\hline
notation & description \\
  \hline
  $\mu_i$ & the $i$-th leverage score of $\M$ \\
  $\hmu_i$ & the approximation of $\mu_i$ with additive error at most $\pm \frac{1}{2 \rho}$ \\
  $\Delta \mu_i $ & equal to $\hmu_i - \mu_i$, which is the error in the estimation \\
  $\mu_i'$ & the $i$-th leverage score of $\W \M$ \\
  $\hmu_i'$ & equal to $\frac{(1-\gamma) \hmu_i}{1 - \gamma \hmu_i}$ \\
  \hline
\end{tabular}
\end{footnotesize}
\end{table}

Here and after, we omit the subscript $i$.
Using Lemma~\ref{lem:weighting}, we have that
the leverage scores after row weighing become
\begin{eqnarray} \label{eq:mu_prime1}
\mu'
& = &
\frac{(1-\gamma) \mu}{1 - \gamma \mu}
\; = \;
\frac{(1-\gamma) (\hmu - \Delta \mu )}{1 - \gamma \hmu + \gamma \Delta \mu} \nonumber\\
& = &
\frac{(1-\gamma) \hmu}{1 - \gamma \hmu}
 \frac{\mu}{\hmu}
\frac{1 - \gamma \hmu}{1 - \gamma \hmu + \gamma \Delta \mu}.
\end{eqnarray}
We define the variable
\begin{eqnarray} \label{eq:def_hmup}
\hmu'
& = &
\frac{(1-\gamma) \hmu}{1 - \gamma \hmu},
\end{eqnarray}
which can be assigned any value in $(0, 1)$.
If $\Delta \mu \ll \hmu, 1 - \hmu$,
then $\mu'$ is close to $\hmu'$.
Equation (\ref{eq:def_hmup}) can be equivalently written as
\begin{eqnarray}
1 - \gamma \hmu
& = & \frac{1 - \hmu}{1 - \hmu'}, \label{eq:gamma1} \\
\gamma
& = & \frac{1 - {\hmu'}/{\hmu}}{1 - \hmu'} \label{eq:gamma2}
\end{eqnarray}
It follows from (\ref{eq:mu_prime1}) (\ref{eq:def_hmup}) (\ref{eq:gamma1}) (\ref{eq:gamma2}) that
\begin{eqnarray*}
\mu'
& = &
\hmu' \frac{\mu}{\hmu}
\frac{1 - \hmu}{1 - \hmu + \Delta \mu (1 - \hmu' / \hmu)} \\
& = &
\hmu' \frac{\mu}{\hmu}
\frac{1 - \hmu}{1 - \mu - \hmu' \Delta \mu / \hmu}.
\end{eqnarray*}
Taking the inverse, we have that
\begin{eqnarray} \label{eq:hmu_bound1}
\frac{1}{\mu'}
& = &
\frac{1}{\hmu'} \frac{\hmu}{\mu} \bigg[  \frac{1-\mu}{1 - \hmu} - \frac{\Delta \mu}{1-\hmu} \frac{\hmu'}{\hmu} \bigg] \nonumber\\
& = &
\bigg[ \frac{1}{\hmu'} - \frac{\Delta \mu}{\hmu (1-\mu)} \bigg] \frac{\hmu}{\mu} \frac{1-\mu}{1-\hmu} .
\end{eqnarray}
By the assumption $\hmu \in \big[\frac{1}{\rho}, 1- \frac{1}{\rho} \big]$,
we have that
\begin{small}
\begin{align*}
&\bigg| \frac{\Delta \mu}{\hmu (1-\mu)} \bigg|^{-1}
\; = \; \frac{1}{|\Delta \mu |} \; \Big|-\hmu^2 + (1+\Delta \mu) \hmu \Big| \\
& \geq \; \frac{1}{|\Delta \mu |} \; \frac{1}{\rho} \bigg( 1- \frac{1}{\rho} \bigg)
    + \min \bigg\{ \sgn (\Delta \mu) \frac{1}{\rho} , \sgn (\Delta \mu) \bigg(1 -  \frac{1}{\rho} \bigg)  \bigg\}\\
& \geq  \;  \bigg( 1- \frac{1}{\rho} \bigg) \bigg(\frac{1}{\rho|\Delta \mu | } - 1 \bigg).
\end{align*}
\end{small}
It follows from $|\Delta \mu| \leq \frac{1}{2\rho}$ that
\begin{eqnarray*}
\bigg| \frac{\Delta \mu}{\hmu (1-\mu)} \bigg|
& \leq &
1 + \frac{1}{\rho - 1} \approx 1.
\end{eqnarray*}
Combined with (\ref{eq:hmu_bound1}) we have
\begin{small}
\begin{align} \label{eq:after_weighing2}
&\bigg[ \frac{1}{\hmu'} - 1 - \frac{1}{\rho - 1} \bigg] \frac{\hmu}{\mu} \frac{1-\mu}{1-\hmu}
 \nonumber\\
&\leq \;\frac{1}{\mu'}
\; \leq
\bigg[ \frac{1}{\hmu'} + 1 + \frac{1}{\rho - 1} \bigg] \frac{\hmu}{\mu} \frac{1-\mu}{1-\hmu}.
\end{align}
\end{small}
We also have that
\begin{eqnarray}
\frac{\mu}{\hmu} \frac{1-\hmu}{1-\mu}
& = & 1 - \frac{\Delta \mu}{ \hmu  (1+\Delta \mu-\hmu) } , \nonumber
\end{eqnarray}
which is less than 1 if $\Delta \mu > 0$ and greater than 1 if $\Delta \mu < 0$.
By the assumption $\hmu \in [\frac{1}{\rho}, 1 - \frac{1}{\rho} \big]$,
when $\Delta \mu \geq 0$, $\frac{\mu}{\hmu} \frac{1-\hmu}{1-\mu}$
has a minimum of $\frac{1- 1/\rho}{2 - 1/\rho} \approx \frac{1}{2}$ which is attained by $\hmu = 1/\rho$ and $\Delta \mu = 1 / 2\rho$;
when $\Delta \mu \leq 0$, $\frac{\mu}{\hmu} \frac{1-\hmu}{1-\mu}$
has a maximum of $\frac{2- 1/\rho}{1 - 1/\rho} \approx 2$ which is attained by $\hmu =1- 1/\rho$ and $\Delta \mu = -1 / 2\rho$.
Thus we have that
\begin{eqnarray}\label{eq:after_weighing3}
\frac{\hmu}{\mu} \frac{1-\mu}{1-\hmu}
& \in & \bigg[\frac{1- 1/\rho}{2 - 1/\rho}, \frac{2- 1/\rho}{1 - 1/\rho}  \bigg].
\end{eqnarray}
Combining (\ref{eq:after_weighing2}) and (\ref{eq:after_weighing3})
we have that
\begin{align*}
& \bigg[ \frac{1}{\hmu'} - 1 - \frac{1}{\rho - 1} \bigg] \frac{1- 1/\rho}{2 - 1/\rho} \\
& \leq \;
\frac{1}{\mu'}
\; \leq
\bigg[ \frac{1}{\hmu'} + 1 + \frac{1}{\rho - 1} \bigg] \frac{2- 1/\rho}{1 - 1/\rho} .
\end{align*}
Setting $\frac{1}{\hmu'} = \frac{\rho-1}{2\rho - 1} \frac{n_1}{k} - \frac{\rho}{\rho - 1} \approx \frac{n_1}{2 k}$, we have that
\[
\mu' \; \in \;
\bigg[  \frac{k}{n_1} , \: \frac{4k}{n_1} \frac{(\rho-0.5)^2}{(\rho-1)^2 - \frac{4k}{n_1} {\rho}({\rho - 0.5})}\bigg].
\]
The scaling parameter $\gamma$ is according to (\ref{eq:gamma2})
taking $\frac{1}{\hmu'} = \frac{n_1}{2 k}$ as input.

\subsection{Proof of Theorems~\ref{thm:mu_large}}

We propose to set $\mu_i = \hmu_i - \frac{1}{2\rho}$ ($\leq \mu_i (\M)$ by the additive error assumption) and $\mu_i' = \frac{1}{\rho}$,
and use (\ref{eq:gamma}) to compute $\gamma$:
\begin{equation}
\gamma \; = \; \frac{1 - \mu_i' / \mu_i }{1 - \mu_i'}
\; = \; \frac{\rho - \frac{1}{\hmu_i - 1/2\rho}}{\rho - 1}.
\end{equation}
For fixed $\gamma$,
the leverage score $\mu_i \big(\W \M\big)$ increases with $\mu_i (\M)$;
since $\mu_i (\M) \geq \mu_i = \hmu_i - \frac{1}{2\rho}$,
we have $\mu_i \big(\W \M\big) \geq \mu_i' = \frac{1}{\rho}$.

\subsection{Proof of Theorem~\ref{thm:approx_cd}}

By Theorems~\ref{thm:mu_medium} and \ref{thm:mu_large},
the choice of $\gamma$ ensures that
\[
\mu_i^\star \leq \mu_i \big(\R^{(t)} \M \big) < \mu_i \big(\R^{(t-1)} \M \big).
\]
Thus the decrement in the objective function follows directly from Lemma~\ref{lem:decrement}.

When $\JM$ is non-empty,
for all $j \in \JM$, $\Delta \mu_{j} > 0$ because ${\mu_{i j} (\R^{(t-1)} \M)} >0$
and $\mu_i \big(\R^{(t-1)}\M\big) - \mu_i \big(\R^{(t)}\M\big) > 0$.
That $\JM$ is non-empty also ensures that $\mu_j \big(\R^{(t-1)}\M\big) < \mu_i^\star$.
Hence the decrement in the objective function (\ref{eq:l1_loss_decrement2}) is strictly greater than zero.

\subsection{Proof Corollary~\ref{cor:decrement}}
The theorem follows directly from Lemma~\ref{lem:monotone}.

\subsection{Leverage Scores}

\begin{theorem} \label{thm:equivalent_leverage}
Let $\A$ be an $n_1 \times n_2$ matrix of rank $r$ ($<n1, n2$),
$\B \in \RB^{n_1\times r}$ be arbitrary bases of $\A$,
and $\R$ be any $n_1\times n_1$ nonsingular matrix.
Then the leverage scores of $\R \A$ and $\R \B$ are the same.
\end{theorem}

\begin{proof}
Let $\U_\A \Si_\A \V_\A^T$ and $\U_\B \Si_\B \V_\B^T$ be 
the condensed SVD of $\A$ and $\B$, respectively. 
Since $\B$ is the bases of $\A$,
there exists an $r\times n_2$ matrix $\C$ such that $\A = \B \C$.
We let $\D = \Si_\B \V_\B^T \C \in \RB^{c\times n_2}$ and obtain
\begin{eqnarray*}
\A & = & \B \C  
\; = \; \U_\B \Si_\B \V_\B^T \C \\
& = & \U_\B \D 
\; = \; (\U_\B \U_\D) \Si_\D \V_\D^T .
\end{eqnarray*}
It is not hard to see that $\U_\B \U_\D = \U_\A \in \RB^{n_1\times r}$.
Let $\u_\A^{(i)}$ and $\u_\B^{(i)}$ be the $i$-th row of $\A$ and $\B$, respectively.
Since $\U_\D$ is an $r\times r$ orthogonal matrix,
we have that 
\[
\big\|\u_\A^{(i)}\big\|_2^2 = \big\|\u_\B^{(i)} \U_\D \big\|_2^2 = \big\|\u_\B^{(i)} \big\|_2^2
\] 
for all $i\in [n_1]$.
The theorem follows by the definition of leverage scores.
\end{proof}

\balancecolumns
\end{document}